\documentclass[acmlarge, nonacm]{acmart}
\makeatletter
\newcommand{\confshort}{\acmConference@shortname}
\newcommand{\conffull}{\acmConference@name}
\newcommand{\confdate}{\acmConference@date}
\newcommand{\confloc}{\acmConference@venue}
\AtBeginDocument{
  \fancypagestyle{firstpagestyle}{
    \fancyhead{}%
    \fancyfoot[C]{}%
  }
  \fancyhf{}
  \fancyhead[LO]{\@headfootfont\shorttitle}%
  \fancyhead[RE]{\@headfootfont\@shortauthors}%
  \fancyhead[LE]{\@headfootfont\footnotesize \confshort, \confdate, \confloc}%
  \fancyhead[RO]{\@headfootfont\footnotesize \confshort, \confdate, \confloc}%
  \fancyfoot[C]{}%
}
\makeatother
\acmBooktitle{\conffull\@ (\confshort), \confdate, \confloc}

\copyrightyear{2026}
\acmYear{2026}
\setcctype{by}
\acmConference[FAccT '26]{The 2026 ACM Conference on Fairness, Accountability, and Transparency}{June 25--28, 2026}{Montr\'eal, Canada}
\acmDOI{10.1145/3805689.3806541}
\acmISBN{979-8-4007-2596-8/2026/06}

\usepackage{algorithm}

\usepackage{amsmath}
\usepackage{booktabs}
\usepackage{amsthm}
\usepackage{array}
\usepackage{xcolor}

\usepackage{dsfont}
\usepackage{amsfonts}
\usepackage{multirow}
\usepackage{wrapfig}

\usepackage{algpseudocode}
\usepackage[most]{tcolorbox}

\newtcolorbox{highlightbox}{
  colback=gray!8,
  colframe=black!35,
  boxrule=0.6pt,
  arc=3pt,
  left=6pt,right=6pt,top=6pt,bottom=6pt
}

\newtheorem{proposition}{Proposition}[section]
\newtheorem{corollary}{Corollary}[section] 

\begin{document}

\title[First-See-Then-Design Framework]
{First-See-Then-Design: A Multi-Stakeholder View for Optimal Performance-Fairness Trade-Offs}

\newcommand{\needcite}{ {\color{pink}[cite]}  }

\newcommand\VRule[1][\arrayrulewidth]{\vrule width #1}

\newcommand{\KG}[1]{[\textcolor{red}{KG :#1}]}
\newcommand{\NK}[1]{[\textcolor{green}{NK :#1}]}
\newcommand{\IV}[1]{\textcolor{violet}{#1}}
\newcommand{\CH}[1]{\textcolor{orange}{CH: #1}]}

\newcommand{\N}{\mathbb{N}}
\newcommand{\R}{\mathbb{R}}
\newcommand{\Q}{\mathbb{Q}}
\newcommand{\Z}{\mathbb{Z}}
\newcommand{\x}{\times}						
\newcommand{\id}{\mathrm{id}}
\newcommand{\Acal}{\mathcal{A}}
\newcommand{\Bcal}{\mathcal{B}}
\newcommand{\Ccal}{\mathcal{C}}
\newcommand{\Dcal}{\mathcal{D}}
\newcommand{\Ecal}{\mathcal{E}}
\newcommand{\Fcal}{\mathcal{F}}
\newcommand{\Gcal}{\mathcal{G}}
\newcommand{\Hcal}{\mathcal{H}}
\newcommand{\Ical}{\mathcal{I}}
\newcommand{\Jcal}{\mathcal{J}}
\newcommand{\Kcal}{\mathcal{K}}
\newcommand{\Lcal}{\mathcal{L}}
\newcommand{\Mcal}{\mathcal{M}}
\newcommand{\Ncal}{\mathcal{N}}
\newcommand{\Ocal}{\mathcal{O}}
\newcommand{\Pcal}{\mathcal{P}}
\newcommand{\Qcal}{\mathcal{Q}}
\newcommand{\Rcal}{\mathcal{R}}
\newcommand{\Scal}{\mathcal{S}}
\newcommand{\Tcal}{\mathcal{T}}
\newcommand{\Ucal}{\mathcal{U}}
\newcommand{\Vcal}{\mathcal{V}}
\newcommand{\Wcal}{\mathcal{W}}
\newcommand{\Xcal}{\mathcal{X}}
\newcommand{\Ycal}{\mathcal{Y}}
\newcommand{\Zcal}{\mathcal{Z}}
\newcommand{\fb}{\mathfrak{b}}

\newcommand{\Abf}{\mathbf{A}}
\newcommand{\Bbf}{\mathbf{B}}
\newcommand{\Cbf}{\mathbf{C}}
\newcommand{\Dbf}{\mathbf{D}}
\newcommand{\Ebf}{\mathbf{E}}
\newcommand{\Fbf}{\mathbf{F}}
\newcommand{\Gbf}{\mathbf{G}}
\newcommand{\Hbf}{\mathbf{H}}
\newcommand{\Ibf}{\mathbf{I}}
\newcommand{\Jbf}{\mathbf{J}}
\newcommand{\Kbf}{\mathbf{K}}
\newcommand{\Lbf}{\mathbf{L}}
\newcommand{\Mbf}{\mathbf{M}}
\newcommand{\Nbf}{\mathbf{N}}
\newcommand{\Obf}{\mathbf{O}}
\newcommand{\Pbf}{\mathbf{P}}
\newcommand{\Qbf}{\mathbf{Q}}
\newcommand{\Rbf}{\mathbf{R}}
\newcommand{\Sbf}{\mathbf{S}}
\newcommand{\Tbf}{\mathbf{T}}
\newcommand{\Ubf}{\mathbf{U}} %
\newcommand{\Vbf}{\mathbf{V}}
\newcommand{\Wbf}{\mathbf{W}}
\newcommand{\Xbf}{\mathbf{X}}
\newcommand{\Ybf}{\mathbf{Y}}
\newcommand{\Zbf}{\mathbf{Z}}

\newcommand{\Mbb}{\mathbb{M}}
\newcommand{\Ibb}{ \mathbb{I} }

\newcommand{\Leb}{\lambda\hspace{-5.3pt}\lambda}

\newcommand{\sm}{\setminus}						
\newcommand{\ins}{\subseteq} 					
\newcommand{\sni}{\supseteq} 					
\newcommand{\cmpl}{\mathsf{c}}

\newcommand{\pr}{\mathrm{pr}}
\newcommand{\ev}{\mathrm{ev}} 
\newcommand{\pf}{\mathrm{pf}} 
\newcommand{\incl}{\mathrm{incl}}
\newcommand{\two}{\mathbf{2}}
\newcommand{\one}{\mathbf{1}}

\newcommand{\srj}{\twoheadrightarrow}
\newcommand{\inj}{\hookrightarrow}
\newcommand{\bij}{\stackrel{\sim}{\longrightarrow}}
\newcommand{\dshto}{\dashrightarrow}
\newcommand{\dshot}{\dashlefttarrow}

\renewcommand{\Pr}{\mathbb{P}} 	
\newcommand{\Var}{\mathrm{Var}}
\newcommand{\Bias}{\mathrm{Bias}}
\newcommand{\Noise}{\mathrm{Noise}}

\newcommand{\E}{\mathbb{E}}
\newcommand{\I}{\mathbbm{1}}
\newcommand{\Pa}{\mathrm{Pa}} 		
\newcommand{\pa}{\mathrm{pa}} 		
\newcommand{\Ch}{\mathrm{Ch}} 		
\newcommand{\Anc}{\mathrm{Anc}} 		
\newcommand{\AnCl}{\mathrm{AnCl}} 
\newcommand{\Desc}{\mathrm{Desc}} 	
\newcommand{\Dist}{\mathrm{Dist}} 
\newcommand{\Pred}{\mathrm{Pred}} 
\newcommand{\Sc}{\mathrm{Sc}}
\newcommand{\NonDesc}{\mathrm{NonDesc}}

\newcommand{\bigdcup}{\mathop{\dot{\bigcup}}}  

\newcommand{\dcup}{\,\dot{\cup}\,}

\newcommand{\ReLU}{\mathrm{ReLU}} 

\newcommand{\lp}{\left ( }
\newcommand{\rp}{\right ) }
\newcommand{\lA}{\left\langle}
\newcommand{\rA}{\right\rangle}
\newcommand{\lB}{\left [ }
\newcommand{\rB}{\right ] }
\newcommand{\lC}{\left \{ }
\newcommand{\rC}{\right \} }
\newcommand{\lI}{\left| }
\newcommand{\rI}{\right| }

\newcommand{\doit}{ {   \operatorname{do} } }

\newcommand*{\tut}[1][]{\mathrel{\tikz [baseline=-0.25ex,-, #1] \draw [#1] (0pt,0.5ex) -- (1.3em,0.5ex);}}
\newcommand*{\tnt}[1][]{\mathrel{\tikz [baseline=-0.25ex,-, #1] \draw [#1] (0pt,0.5ex) -- node[strike out,draw,-]{} (1.3em,0.5ex);}}
\newcommand*{\tuh}[1][]{\mathrel{\tikz [baseline=-0.25ex,-latex, #1] \draw [#1] (0pt,0.5ex) -- (1.3em,0.5ex);}}
\newcommand*{\hut}[1][]{\mathrel{\tikz [baseline=-0.25ex,latex-, #1] \draw [#1] (0pt,0.5ex) -- (1.3em,0.5ex);}}
\newcommand*{\huh}[1][]{\mathrel{\tikz [baseline=-0.25ex,latex-latex, #1] \draw [#1] (0pt,0.5ex) -- (1.3em,0.5ex);}}
\newcommand*{\ouo}[1][]{\mathrel{\tikz [baseline=-0.25ex,-, #1] \draw [
decoration={markings, mark=at position 0 with {\draw circle (1pt);}, 
mark=at position 1 with {\draw circle (1pt);}}, postaction={decorate},#1] (0pt,0.5ex) -- (1.3em,0.5ex);}}
\newcommand*{\tuo}[1][]{\mathrel{\tikz [baseline=-0.25ex,-, #1] \draw [
decoration={markings, mark=at position 1 with {\draw circle (1pt);}}, postaction={decorate},#1] (0pt,0.5ex) -- (1.3em,0.5ex);}}
\newcommand*{\huo}[1][]{\mathrel{\tikz [baseline=-0.25ex,latex-, #1] \draw [
decoration={markings, mark=at position 1 with {\draw circle (1pt);}}, postaction={decorate},#1] (0pt,0.5ex) -- (1.3em,0.5ex);}}
\newcommand*{\out}[1][]{\mathrel{\tikz [baseline=-0.25ex,-, #1] \draw [
decoration={markings, mark=at position 0 with {\draw circle (1pt);}}, postaction={decorate},#1] (0pt,0.5ex) -- (1.3em,0.5ex);}}
\newcommand*{\ouh}[1][]{\mathrel{\tikz [baseline=-0.25ex,-latex, #1] \draw [
decoration={markings, mark=at position 0 with {\draw circle (1pt);}}, postaction={decorate},#1] (0pt,0.5ex) -- (1.3em,0.5ex);}}
\newcommand*{\toh}[1][]{\mathrel{\tikz [baseline=-0.25ex,-latex, #1] \draw [
decoration={markings, mark=at position 0.4 with {\draw[fill] circle (0.8pt);}},
        postaction={decorate},#1] (0pt,0.5ex) -- (1.3em,0.5ex);}}
\newcommand*{\hot}[1][]{\mathrel{\tikz [baseline=-0.25ex,latex-, #1] \draw [
decoration={markings, mark=at position 0.6 with {\draw[fill] circle (0.8pt);}},
        postaction={decorate},#1] (0pt,0.5ex) -- (1.3em,0.5ex);}}
\newcommand*{\tot}[1][]{\mathrel{\tikz [baseline=-0.25ex,-, #1] \draw [
decoration={markings, mark=at position 0.5 with {\draw[fill] circle (0.8pt);}},
        postaction={decorate},#1] (0pt,0.5ex) -- (1.3em,0.5ex);}}				
\newcommand*{\hoh}[1][]{\mathrel{\tikz [baseline=-0.25ex,latex-latex, #1] \draw [
decoration={markings, mark=at position 0.5 with {\draw[fill] circle (0.8pt);}},
        postaction={decorate},#1] (0pt,0.5ex) -- (1.3em,0.5ex);}}

\newtheorem{sa}{Theorem}[section]
\newtheorem{Thm}[sa]{Theorem}
\newtheorem{Lem}[sa]{Lemma}
\newtheorem{Prp}[sa]{Proposition}
\newtheorem{Cor}[sa]{Corollary}
\newtheorem{Con}[sa]{Conjecture}
\newtheorem{Fct}[sa]{Facts}
\newtheorem{Prn}[sa]{Principle}

\newtheorem{Def}[sa]{Definition}
\newtheorem{DefLem}[sa]{Definition/Lemma}
\newtheorem{Axm}[sa]{Axiom}
\newtheorem{Not}[sa]{Notation}
\newtheorem{Asm}[sa]{Assumptions}

\newtheorem{Rem}[sa]{Remark}
\newtheorem{Cau}[sa]{Caution}
\newtheorem{Eg}[sa]{Example}
\newtheorem{Tho}[sa]{Thoughts}

\newcommand{\TightenPar}[1]{\looseness=-#1}

\newtheorem{remark}{Remark}

\author{Kavya Gupta}
\email{gupta@cs.uni-saarland.de}
\affiliation{%
  \institution{ Department of Computer Science, Saarland University}
  \city{Saarbrücken}
  \country{Germany}
}

\author{Nektarios Kalampalikis}
\email{kalampalikis@cs.uni-saarland.de}
\affiliation{%
  \institution{ Department of Computer Science, Saarland University}
  \city{Saarbrücken}
  \country{Germany}}

\author{Christoph Heitz}
\email{christoph.heitz@zhaw.ch}
\affiliation{%
  \institution{Zurich University of Applied Sciences}
  \city{Zurich}
  \country{Switzerland}
}

\author{Isabel Valera}
\email{ivalera@cs.uni-saarland.de}
\affiliation{%
 \institution{ Department of Computer Science, Saarland University}
  \city{Saarbrücken}
  \country{Germany}}

\renewcommand{\shortauthors}{Gupta et al.}

\begin{abstract}
Fairness in algorithmic decision-making is often defined in the predictive space, where predictive performance — used as a proxy for decision-maker (DM) utility — is traded off against prediction-based fairness notions, such as demographic parity or equality of opportunity. This perspective, however, ignores how predictions translate into decisions and ultimately into utilities and welfare for both DM and decision subjects (DS), as well as their allocation across social-salient groups.

In this paper, we propose a \emph{multi-stakeholder framework for fair algorithmic decision-making} grounded in welfare economics and distributive justice, explicitly modeling the utilities of both the DM and DS, and defining fairness via a social planner’s utility that captures inequalities in DS utilities across groups under different justice-based fairness notions (e.g., Egalitarian, Rawlsian). We formulate fair decision-making as a post-hoc multi-objective optimization problem, characterizing the achievable performance-fairness trade-offs in the two-dimensional utility space of DM utility and the social planner's utility, under different decision policy classes (deterministic vs. stochastic, shared vs. group-specific). Using the proposed framework, we then identify conditions (in terms of the stakeholders' utilities) under which stochastic policies are more optimal than deterministic ones, and empirically demonstrate that simple stochastic policies can yield superior performance–fairness trade-offs by leveraging outcome uncertainty. Overall, we advocate a shift from prediction-centric fairness to a \emph{transparent, justice-based, multi-stakeholder approach that supports the collaborative design of decision-making policies}.
\end{abstract}


\maketitle
\noindent\textbf{Copyright notice.} © ACM 2026. This is the author's version of the work. 
It is posted here for your personal use. Not for redistribution. 
The definitive Version of Record was published in 
Proceedings of the 2026 ACM Conference on Fairness, Accountability, and Transparency (FAccT '26), 
\url{https://doi.org/10.1145/3805689.3806541}.
\section{Introduction}
Fairness analyses in algorithmic decision-making are most often framed through the lens of classification~\cite{barocas2023fairness,friedler2021possibility}, where  the performance and the (un)fairness of an algorithm (in this case, a classifier) are measured in terms of its predictions. For example, the performance of an algorithm is often measured in terms of accuracy, while its fairness is assessed in terms of disparity between, e.g., positive rates (i.e., demographic parity) or true positive rates (i.e., equality of opportunity)~\cite{hardt2016equality,zafar2017fairness,pleiss2017fairness} across socio-demographic groups.  
In this paper, we refer this approach to fairness as \emph{prediction-centric} as it is 
characterized by two key characteristics: 
(i) the predictions of the classifier are treated as proxies for the decision-maker's utility (e.g., profit, risk minimization); and,  
ii) both performance and fairness are defined and optimized in the \emph{predictive space}, i.e., in terms of the algorithm’s predictions rather than on the downstream consequences of the decisions made using those predictions. 
Predictions are first transformed into decisions through a decision policy, and these decisions then induce utilities and welfare outcomes for affected individuals. %
By collapsing this prediction-decision-utility pipeline \cite{beigang2022advantages}  into predictive metrics, prediction-centric approaches implicitly assume that fairness in predictions translates into fairness in outcomes,  an assumption that does not generally hold \cite{scantamburlo2025prediction,liang2024algorithmic}.

A growing body of work has challenged this perspective, advocating instead for the adoption of a utility-based fairness framework grounded in welfare economics~\cite{heidari2018fairness,lee2021formalising} and distributive justice~\cite{rawls2017theory,roemer2015equality,lamont2017distributive,hertweck2023justice}. These approaches argue that fairness should be assessed in terms of outcomes, that is, how decisions distribute benefits and harms across individuals and groups, rather than in terms of predictive proxies.
 Accordingly, fairness interventions should target disparities in the utilities induced by decision-making policies among social-salient groups, rather than disparities in predictions alone.
As emphasized in \cite{rosenfeld2025machine}, the focus should be on maximizing stakeholders’ utilities induced by the decision-making policy, rather than solely the predictive accuracy of the underlying machine learning model: What ultimately matters is how a decision-making policy allocates utilities among affected stakeholders. 
Inspired by these insights, we propose \textbf{a multi-stakeholder framework for fair algorithmic decision-making}, illustrated in Figure~\ref{fig:MOO_view}. 

First, we advocate for explicitly defining and empirically measuring the utilities associated with all relevant stakeholders in the decision-making pipeline, which we assume here to be:
(i) \textit{the decision maker (DM)}, whose objective is to deploy a decision policy that maximizes its utility function ($U_{DM}$) (e.g., expected profit);
(ii) \textit{the decision subjects (DS)}, individuals affected by the decisions, who experience outcomes through their own utility function ($U_{DS}$); and
(iii) \textit{a social planner (SP)}, whose objective is to promote social welfare, represented by a utility function ($U_{SP}$) that quantifies the (un)fairness of a decision policy through the inequality in utilities experienced by DS across different social-salient groups. The SP operationalizes fairness by aggregating or comparing the utilities of DS across social-salient groups according to a chosen theory of justice (e.g., Egalitarian~\cite{arneson2002egalitarianism}, Rawlsian~\cite{rawls2001justice,rawls2017theory} principles). %
\emph{By construction, this formulation directly integrates fairness into the utility space.} \TightenPar{1} %

Second,  we build on this stakeholder decomposition to adopt a {post-hoc multi-objective optimization (MOO)~\cite{kang2024survey}} perspective, jointly optimizing the utilities of the DM and the SP to characterize the achievable Pareto front in the {utility space}, thereby \emph{making the trade-off between performance (captured by $U_{DM}$) and fairness (captured by $U_{SP}$) explicit and transparent} to all stakeholders. 
Crucially, this post-hoc MOO analysis allows stakeholders to first examine how different policies %
affect the attainable trade-offs under different fairness notions. 
This allows stakeholders to make informed decisions about which decision to deploy, based on their normative commitments, regulatory constraints, and operational priorities.

\textbf{Note that the goal of our framework is not to prescribe a single “optimal” policy, but rather to expose the full set of Pareto-optimal alternatives \cite{wei2022fairness}, thus enabling an informed and transparent discussion among stakeholders, who can then jointly decide which decision policy to deploy.} 

\begin{figure}[h!]
\begin{minipage}{0.4\textwidth}
     \centering    \includegraphics[width=0.9\linewidth]{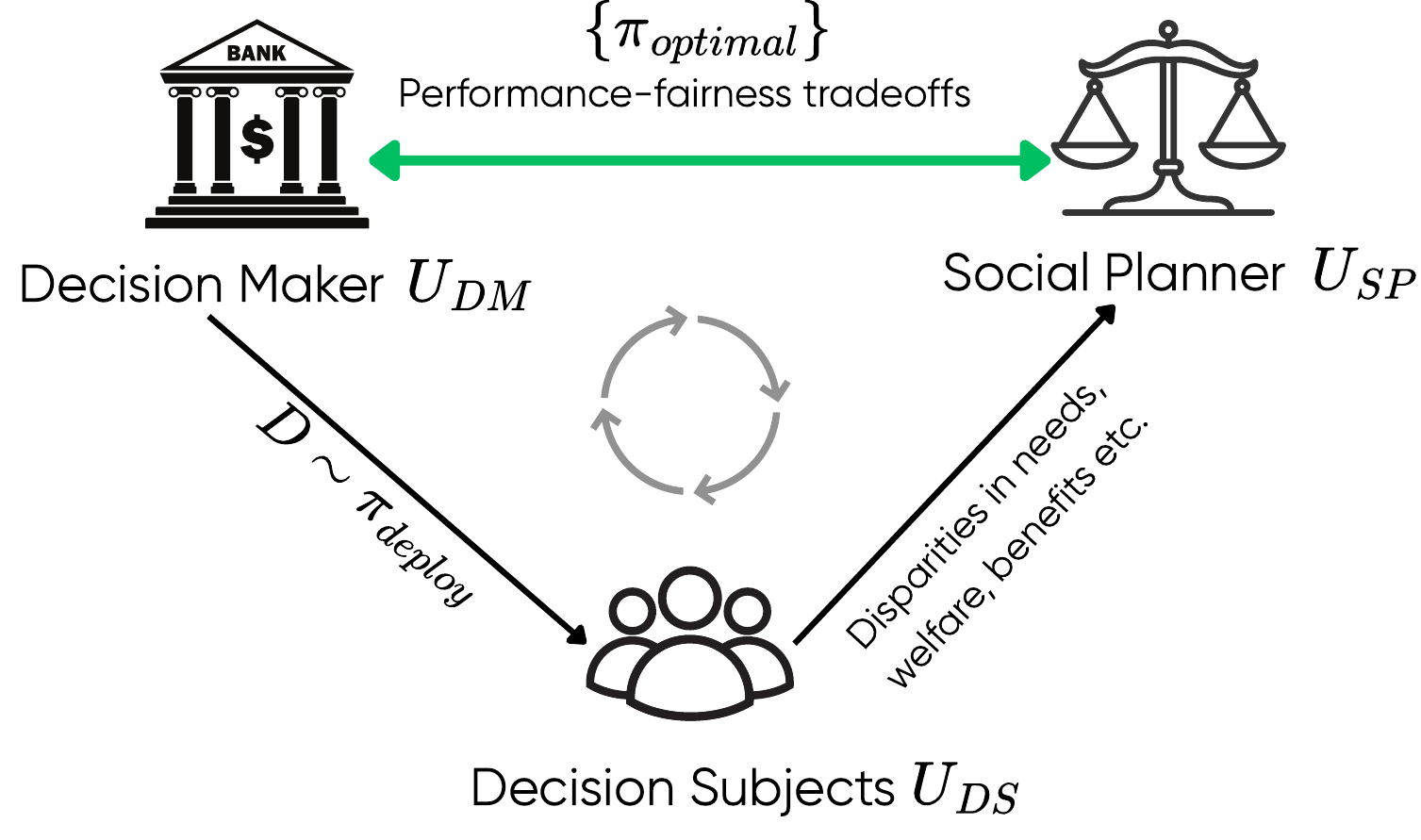}
    \vspace{-7pt}\caption{%
    Multi-stakeholder view of decision-making policies.
    DM deploys policy $\pi$ to make decisions aiming to maximize its utility $U_{DM}$. The decisions affect DS, whose welfare $U_{DS}$ may vary across social groups. The SP evaluates policies based on disparities in group-level welfare, contributing to a distinct utility $U_{SP}$.
    }
    \label{fig:MOO_view}
\end{minipage}\hfill
   \begin{minipage}{0.57\textwidth}
   \centering
    \includegraphics[width=1.01\linewidth]{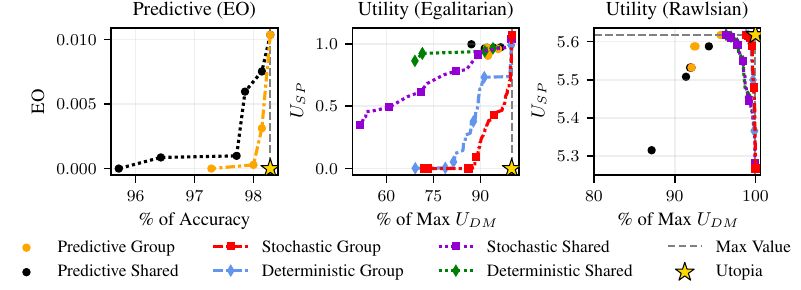}
    \caption{Pareto fronts for different policy classes in predictive space (left) and utility space (right) for different notions of distributive justice. Projected predictive-space policies are suboptimal in utility space %
    whereas deterministic and stochastic utility-space policies achieve better trade-offs.}
    \label{fig:introgerman_plot}
    \end{minipage}
\end{figure}

\subsection{Motivating Example} 
We illustrate the need for the proposed multi-stakeholder, utility-based approach to algorithmic fairness with an example in the context of lending decisions in the German credit dataset~\cite{hofmann1994statlog}. 
Figure~\ref{fig:introgerman_plot} first presents the trade-offs in the predictive space between predictive performance (accuracy) and a prediction-based fairness notion~\cite{nagpal2025optimizing,liang2022algorithmic}, here equality of opportunity. We consider two families of threshold-based decision policies: shared policies, which do not condition on sensitive attributes at decision time, and group-specific policies, which do. For each policy family, Figure~\ref{fig:introgerman_plot} (left) shows the predictive Pareto front. 
We then project the predictive Pareto-optimal policies into the utility space and compare them with policies obtained by directly optimizing utilities.
We instantiate $U_{DM}$ to measure the expected profit as difference between the expected simple interest and the loss given default,  and adopt two notions of justice-based fairness to measure $U_{SP}$: a) \emph{Egalitarian}, thus measuring disparities in $U_{DS}$ across social-salient groups; and b) \emph{Rawlsian}, focusing on the minimum $U_{DS}$ across social-salient groups. The utopia point denotes the theoretical (even though generally unattainable) target of the two-dimensional optimization. 
This comparison yields three key observations:

    \vspace{-7pt}
    \paragraph{ \textbf{-- Predictive trade-offs do not {necessarily} transfer to utility space.}} Policies that are Pareto-optimal in the predictive space occupy a small, often suboptimal, region of the utility space. 
    Instead, directly optimizing utilities yields  substantially richer Pareto fronts, expanding the set of viable policies available at deployment. {Importantly, we do not view prediction-centric fairness notions as flawed: for certain configurations, our framework recovers standard predictive fairness criteria. Rather, predictive metrics capture only part of the decision pipeline, and their implications depend on downstream utility effects.} This perspective complements prior work highlighting the mismatch between predictive metrics and downstream welfare outcomes~\cite{barocas2023fairness, mehrabi2021survey,chouldechova2017fair,obermeyer2019dissecting}.
    
    \vspace{-7pt}
   \paragraph{ \textbf{-- The nature of trade-offs depends on the chosen notion of justice and utility design.}} The existence and shape of performance-fairness trade-offs are highly sensitive to how utilities are defined and how fairness is operationalized (here, Egalitarian versus Rawlsian). In our example, under a Rawlsian notion of fairness, group-specific policies can achieve near-maximal decision-maker utility (approximately 98\% of the maximum). Under the Egalitarian notion, however, achieving zero unfairness requires a substantially larger reduction in decision-maker utility (below 75\% of its maximum). This shows that imposing fairness constraints \emph{a priori} without first examining the attainable trade-offs can be misleading. 
    
    \vspace{-7pt}
    \paragraph{ \textbf{-- Performance-fairness trade-offs  depends on the policy parameterization.} }
    The achievable trade-offs are strongly affected by how decision policies are parameterized. In our example, shared policies are unable to achieve near zero Egalitarian unfairness without substantial loss in decision-maker utility, whereas group-specific policies can. Moreover, under Egalitarian fairness, stochastic policies strictly dominate deterministic threshold-based policies,  %
    improving fairness for a given level of decision-maker utility. These results thus support our claims on the need of making the performance-fairness trade-offs transparent prior to the design of the deployed policy.

Although these observations may initially seem anecdotal, we show in the remainder of the paper, both in our theoretical and empirical results,
 that their implications are in fact systematic and substantive.

\pagebreak
\subsection{Related Work}

Recent work highlights that fairness in decision-making should consider multiple stakeholders, each with potentially conflicting utilities, such as decision-makers, individuals, and regulators~\cite{heidari2019moral, hertweck2023justice}. This motivates the use of MOO as a principled framework. MOO can surface trade-offs between performance and fairness, rather than enforcing fairness as a hard constraint on a single objective~\cite{zafar2017fairness, martinez2020minimax}. Beyond parity-based metrics, research emphasizes utility-based and justice-aware fairness, focusing on real-world impacts and the moral claims of affected individuals~\cite{heidari2018fairness, hertweck2023justice}. Recent frameworks explicitly connect algorithmic fairness to distributive justice principles, such as Egalitarianism, Rawlsian max–min fairness, and prioritarianism~\cite{hertweck2023justice, casacuberta2023augmenting}. These frameworks highlight the normative assumptions in standard fairness metrics. Empirical studies show that enforcing parity-based constraints, without considering heterogeneous utilities, can unintentionally harm marginalized populations~\cite{hu2020fair, corbett2017algorithmic}. 
Randomization has been shown to be necessary for satisfying group fairness constraints in classification~\cite{celis2019classification}, exposure and allocation problems~\cite{wang2021stochastic}, selective label settings~\cite{kilbertus2020fair, wei2021decision}, and bandit learning~\cite{joseph2016fairness}. Our work builds on these insights by providing a unified geometric characterization of when and why stochastic policies expand attainable performance-fairness trade-offs in static decision-making. Unlike prior work, we focus on post-hoc, utility-space Pareto analysis, showing how stochasticity interacts with stakeholder utilities and justice notions to shape the geometry of achievable trade-offs. A more detailed discussion of related work is provided in Appendix~\ref{app:related_work}.

\section{Background}\label{sec:background}
\subsection{Decision-maker Optimal decision-making policy}
We adopt the perspective of a decision maker 
whose objective is to maximize expected utility. Following~\cite{kilbertus2020fair, corbett2017algorithmic}, we define the DM’s utility under a decision policy $\pi$ as the expected utility induced by the decisions made using that policy:
\begin{align}
\label{eqn:udm}
U_{DM}(\pi) 
&= \mathbb{E}_{(X, S, Y) \sim \mathbb{P}, \, D \sim \pi(\cdot|X, S)} 
\left[ u_{DM}(Y, X, D) \right] \nonumber \\
&= \mathbb{E}_{(X, S)\sim \mathbb{P}} \Big[ \pi(D{=}1 \mid X, S)  \left( 
             \mathbb{P}(Y{=}1 \mid X, S) C_{11}(X) \right. \nonumber
 - \left. \mathbb{P}(Y{=}0 \mid X, S) C_{10}(X) \right) \\
&\quad + (1 - \pi(D{=}1 \mid X, S))\left( 
             \mathbb{P}(Y{=}0 \mid X, S) C_{00}(X) \right.  
 - \left. \mathbb{P}(Y{=}1 \mid X, S) C_{01}(X) \right) \Big], 
\end{align}

where $S$ denotes sensitive attributes and \( C_{DY}(X) \) denotes the DM's cost or benefit associated with decision \( D \in \{0,1\} \) and outcome \( Y \in \{0,1\} \), potentially dependent on the non-sensitive features \( X \). This formulation generalizes standard prediction accuracy: if we assume symmetric costs (e.g., \( C_{11} = C_{00} = 1 \), \( C_{01} = C_{10} = 0 \)), then maximizing \( U_{DM} \) reduces to maximizing classification accuracy. More generally, the formulation captures asymmetries in false positives/negatives, as well as individual-specific or context-dependent utilities, which are central in many real-world decision problems such as lending~\cite{moscato2021benchmark}, hiring~\cite{mahmoud2019performance}, criminal justice~\cite{dieterich2016compas} or medical triage~\cite{habehh2021machine}.

Assuming access to the true conditional outcome probability \( \mathbb{P}(Y = 1 \mid X, S) \), the DM-optimal decision policy is given by a deterministic threshold rule:
\begin{equation}
    \pi^*_{DM}(X, S) = \mathds{1}\left\{ \mathbb{P}(Y = 1 \mid X, S) \geq \tau(X) \right\},
\end{equation}
with the optimal threshold being:
\begin{equation}
    \tau(X) = \frac{C_{00}(X) + C_{10}(X)}{C_{00}(X) + C_{11}(X) + C_{01}(X) + C_{10}(X)}.
\end{equation}
This result highlights a key distinction between prediction and decision-making: even when the prediction model is perfectly calibrated, optimal decisions depend critically on the underlying utility structure, not solely on predictive accuracy. In practice, since the true conditional distribution \( \mathbb{P}(Y \mid X, S) \) is unknown, it is approximated via a binary soft classifier \( h(X, S) \in [0,1] \), typically trained to minimize prediction loss. This motivates the following class of decision policies\TightenPar{1}:

\begin{definition}[Deterministic Threshold Policy]
    A \emph{deterministic decision policy} is defined as:
    \begin{equation}
         \pi_\tau(X, S) = \mathds{1}\{ h(X, S) \geq \tau \}, \quad \tau \in [0,1],
    \end{equation}
    where \( h(X, S) \) is a calibrated soft classifier and $\tau$ is the decision threshold that converts the score into a binary action.
\end{definition}

Such policies are optimal from the DM’s perspective given the estimated predictive probabilities and utility structure. However, as we discuss next, DM-optimal policies may induce substantial disparities in outcomes across sensitive groups particularly when label distributions, error rates, or utilities differ across groups~\cite{zafar2017fairness}. This observation motivates the need for multi-stakeholder and justice-based extensions beyond DM-optimal decision-making.

\subsection{Prediction-centric Fairness}
To mitigate disparities induced by DM-optimal decision-making, fairness-aware learning typically enforces constraints on group-level classification behavior. Among the many group fairness notions proposed in the literature, Equality of Opportunity (EO)~\cite{hardt2016equality} is particularly prominent in high-stakes decision-making settings, as it aims to ensure fair treatment of qualified individuals across sensitive groups.

\begin{definition}[Equality of Opportunity (EO)]
    A decision policy \( \pi \) satisfies EO if $\forall s, s' \in \mathcal{S}$:
    \[
        \mathbb{E}[ D = 1 \mid S = s, Y = 1 ] = \mathbb{E}[ D = 1 \mid S = s', Y = 1 ]
    \]
\end{definition}

EO enforces equality of true positive rates across groups, thereby constraining how decisions are allocated among individuals with positive outcomes. In practice, EO is often measured via the absolute disparity in true positive rates:
$$\text{EO}(\pi) = \left| \mathbb{E}[D=1 \mid S=s, Y=1] - \mathbb{E}[D=1 \mid S=s', Y=1] \right|$$

and enforced through either in-processing or post-processing interventions.
 In-processing approaches incorporate fairness constraints directly into the learning objective, encouraging the classifier to balance group-level error rates during training~\cite{zafar2017fairness}. In contrast, post-processing approaches operate on a trained classifier by applying group-dependent decision rules to equalize true positive rates without modifying the underlying predictive model~\cite{hardt2016equality, corbett2017algorithmic}. 

Within this framework, we distinguish between two classes of decision policies. A \emph{shared} policy uses a common single threshold $\tau$ and classifier 
$h$ across all groups, treating all individuals identically regardless of the sensitive attributes. In contrast, a \emph{group-specific} policy allows each group 
$s\in \mathcal{S}$ to have its own threshold $\tau_s$, enabling decisions to adapt to group-level characteristics in order to satisfy fairness constraints.

While optimizing for EO balances group-level error rates, it operates entirely in the \emph{predictive space}, focusing on statistical properties of predictions and induced decisions rather than their downstream consequences. As a result, prediction-centric fairness notions may fail to capture disparities in individual utility and well-being~\cite{baumann2022distributive, corbett2023measure} (Figure~\ref{fig:introgerman_plot}). For example, equalizing true positive rates does not guarantee that different groups derive comparable benefits from favorable decisions when utilities differ across individuals or groups.
In the next section, we address these limitations by formalizing utility-based fairness, where decision policies are evaluated and optimized directly in the utility space, accounting for the heterogeneous welfare impacts of decisions across stakeholders.

\section{A Multi-Stakeholder Approach to Algorithmic Fairness}
\label{sec:methodology}

\subsection{Justice-based Fairness and Stakeholder Utilities}

Justice-based fairness shifts the focus from predictive metrics to the \textbf{utilities induced by decisions} and the resulting welfare experienced by stakeholders. While prediction-based metrics act as proxies for impact, they often obscure the real-world consequences of decisions. Justice-based approaches instead evaluate fairness in terms of how decision policies distribute benefits, harms  and welfare across individuals and groups~\cite{hertweck2024group}. In this work, we make this shift explicit by modeling stakeholder utilities directly and defining fairness objectives grounded in utility and distributive justice. \TightenPar{1}
\paragraph{Decision Subject Utility.} Let \( u_{DS}: \{0,1\} \times \{0,1\} \rightarrow \mathbb{R} \) denote the per-instance utility experienced by the decision subject, where \( u_{DS}(y, d) \) captures the benefit or harm of decision \( d \) given outcome \( y \). The expected utility of decision subjects belonging to a group \( s \in \mathcal{S} \) under decision policy \( \pi \) is:
\begin{equation}
U_{DS}^s(\pi)
= \mathbb{E}_{(X,Y)\sim \mathbb{P}(\cdot\mid S=s),\; D\sim \pi(\cdot\mid X,S=s)}
\big[ u_{DS}(Y, D) \big].
\end{equation}
making explicit that decision-subject utilities may differ systematically across socio-demographic groups. Such heterogeneity is common in practice. 
For example, the rejection of a qualified applicant by a hiring or lending system may impose substantially higher costs on individuals from under-resourced or historically disadvantaged groups, who often have fewer alternative opportunities. In contrast, the same decision may be less harmful to individuals with greater access to resources or support networks. These asymmetries highlight why fairness should be evaluated in terms of \emph{realized utility} rather than parity of predictions or error rates~\cite{hu2020fair, baumann2022distributive}.

\paragraph{Fairness as Regulator (Social Planner) Utility} We model fairness from the perspective of a regulator or social planner who evaluates decision policies based on how utilities are distributed across groups. Let $\mathbf U_{DS}(\pi)=(U_{DS}^1(\pi),\dots,U_{DS}^{|\mathcal S|}(\pi))$, denote the vector of group-wise decision-subject utilities induced by policy 
$\pi$, the social planner’s utility is defined as,
\begin{equation}
U_{SP}(\pi) = \mathcal R\left( \mathbf U_{DS}(\pi) \right),
\end{equation}
where $\mathcal R$ encodes a notion of distributive justice and quantifies the (un)fairness of a policy through inequalities in group utilities. This formulation captures the idea that fairness is not an intrinsic property of predictions, but a normative evaluation of outcome distributions~\cite{rambachan2020economic}.
Different theories of justice correspond to different choices of $\mathcal R$. Examples include\TightenPar{1}:
\begin{itemize}
    \item \textbf{Egalitarianism}~\cite{arneson2002egalitarianism}, which penalizes inequality in utilities, $
    \mathcal R_{\mathrm{Egal}}(\mathbf U_{DS}(\pi)) = 
    \sum_{s,s'} |U_{DS}^s(\pi) - U_{DS}^{s'}(\pi)|,
    $
    \item \textbf{Rawlsian max–min fairness}~\cite{rawls2001justice, rawls2017theory}, that prioritizes the worst-off group,
$\mathcal R_{\mathrm{Rawls}}(\mathbf U_{DS}(\pi)) = \min_s U_{DS}^s(\pi).$
\end{itemize}
While other justice principles such as prioritarianism~\cite{holtug2017prioritarianism} or sufficientarianism~\cite{shields2020sufficientarianism} can also be accommodated, our theoretical and empirical analysis focuses on Egalitarian and Rawlsian objectives as canonical and widely adopted formulations. Importantly, in the remainder of the paper we treat $\mathcal R$ abstractly and require only mild regularity conditions such as convexity or concavity which play a central role in characterizing the geometry of achievable performance-fairness trade-offs, which is central to our theoretical analysis. 

\subsection{Performance-Fairness Trade-offs: Formalizing the Gap}
We formalize performance-fairness trade-offs in a multi-stakeholder framework over decision policies. This section establishes a common formal language for comparing predictive and utility-based notions of optimality, defining the policy classes, objective spaces, Pareto fronts, and projections that will be analyzed theoretically in the next section. \TightenPar{1}
\begin{definition}[Policy Class]
    Let $\Pi$ denote the class of %
    decision policies induced by a scoring function $h$, 
where $h:(X\times S)\to[0,1]$ is a (possibly group-aware) score.
Each policy $\pi \in \Pi$ is evaluated according to a two-dimensional objective $f(\pi) = (P(\pi),F(\pi)) \in \mathbb{R}^2$, where 
$P(\pi)$ measures performance and $F(\pi)$ measures fairness. Both objectives are assumed to be maximized.
\end{definition}

\begin{definition}[Pareto Optimality]
Given two objective vectors $z,z'\in \mathbb{R}^2$, we say that $z$ strictly Pareto-dominates $z'$, denoted $z \succ z'$, if
$z_i \ge z'_i$ for all $i\in\{1,2\}$ and there exists $j\in\{1,2\}$ such that $z_j > z'_j$.
\end{definition}

\begin{definition}[Pareto-optimal set (PS)]
A policy $\pi \in \Pi$ is Pareto-optimal if no other policy strictly dominates it: $PS(f,\Pi)=\{ \pi \in \Pi \mid \nexists \pi' \neq \pi \quad s.t. \quad f(\pi') \succ f(\pi) \}$.
\end{definition}

\begin{definition}[Pareto Front (PF)]
The corresponding Pareto front is the image of this set in objective space: $PF(f, \Pi) = \{ f(\pi) \mid \pi \in PS(f,\Pi)\}$.  Each point on the Pareto front represents an efficient performance-fairness trade-off, in the sense that improving one objective necessarily degrades the other. The front therefore defines the boundary of
efficient policies from which a decision-maker or social planner may select according to stakeholder priorities.\TightenPar{1}
\end{definition}

\subsection{Predictive Space vs. Utility Space}
We consider two distinct objective spaces depending on how performance and fairness are defined:
\begin{itemize}
\item \textbf{Predictive space:} $f^{pred}(\pi)=(Acc(\pi),EO(\pi))$,
where performance and fairness are measured using predictive metrics, with $PS^{pred} = PS(f^{pred},\Pi),\quad PF^{pred}=PF(f^{pred},\Pi)$ for policy class $\Pi$.%
\item \textbf{Utility space:} $f^{util}(\pi)=(U_{DM}(\pi),U_{SP}(\pi))$,
where performance and fairness are evaluated directly in terms of stakeholder utilities, with $PS^{util} = PS(f^{util},\Pi),\quad PF^{util}=PF(f^{util},\Pi)$.
\end{itemize}

\paragraph{Projecting Predictive-Optimal Policies into Utility Space. }
Policies that are Pareto-optimal in predictive space induce a corresponding set of utility outcomes. {In practice, this projection is an evaluation step: we first identify the Pareto-optimal policies in predictive space, each corresponding to a specific parameter configuration, and then evaluate these same policies under the utility objectives while keeping their parameters fixed.}
We define the projected predictive set as: 
$
    S_{proj} =\{f^{util}(\pi)\mid \pi \in PS^{pred}\}
$.
Since policies in $PS^{pred}$ are not optimized for utility-based objectives, points in $S_{proj}$ need not be Pareto-optimal in utility space. To enable a meaningful comparison, we define the projected predictive Pareto front:
$ PF^{pred}_{proj}=PF(f^{util},S_{proj})$
which represents the best utility-space trade-offs achievable among policies that are predictive-optimal (Figure~\ref{fig:introgerman_plot}).  \TightenPar{1} 

\subsection{Expanding the Policy Class: Stochastic Policies}
To enlarge the set of attainable trade-offs, we consider a richer class of simple stochastic decision policies, that allow for probabilistic decision-making near the
threshold i.e.,
\begin{equation}
    \Pi_{\text{stoch}} = \left\{ \pi_{\beta, \gamma}(X,S) = \sigma\left( \beta (h(X,S) - \gamma) \right) \mid \beta \in \mathbb{R}_+, \gamma \in [0,1] \right\},
\end{equation}

where $\sigma(\cdot)$ is the sigmoid function. The deterministic class $\Pi_{\text{det}}$ is recovered in the limit \( \beta \to \infty \), and hence \( \Pi_{\text{det}} \subset \Pi_{\text{stoch}} \). 
Each stochastic policy induces expected utilities through randomized decisions, then:
\[
\mathbf f^{\text{util}}(\pi_{\beta,\gamma})
=\big(U_{DM}(\pi_{\beta,\gamma}),\ U_{DS}(\pi_{\beta,\gamma})\big).
\]

Let $PF^{util}_{det}$ and $PF^{util}_{stoch}$ denote the utility-space Pareto fronts under deterministic and stochastic policy classes, respectively. Then by construction, stochastic policies weakly dominate deterministic ones, that is:
 \begin{equation}
     PF^{util}_{stoch} \succeq PF^{util}_{det}
 \end{equation}
 with strict dominance arising when randomization enables trade-offs that are unattainable by deterministic thresholds this depends on the specific data distribution and utility definitions.
In Section~\ref{sec:theory}, we theoretically identify the settings where this
strict improvement holds.

\section{Understanding and Improving Utility-Space Trade-offs}
\label{sec:theory}
This section provides a theoretical analysis of performance-fairness trade-offs in the utility space. We study how the structure of stakeholder utilities and the choice of decision policy class shape the geometry of attainable trade-offs, and when stochastic decision-making can strictly improve the set of Pareto-optimal outcomes. Again as a reminder, we focus on the consequences of decisions, not on predictions.

\subsection{Utility Representation and Stakeholder Alignment}
We consider binary decisions $D \in \{0,1\}$ made on individuals characterized by features $X$ and sensitive attribute 
$S$, followed by the realization of an outcome $Y \in \{0,1\}$. Decisions precede outcomes, and utilities are induced by the joint realization of $(D,Y)$. Any (possibly stochastic) decision policy $\pi$ induces, for each sensitive group $s$, a joint distribution $p_{\pi}(d,y\mid s)=P_{\pi}(D=d,Y=y \mid S=s)$, obtained by marginalizing over the feature distribution 
$X$ and the randomness of the policy. This distribution captures how frequently each decision-outcome pair occurs for group 
$s$ under 
$\pi$, and thus summarizes the downstream consequences of the policy at the population level. 
We assume that both DM and DS utilities can be expressed as affine functions of these joint decision-outcome probabilities, i.e.,
\begin{equation}
\begin{aligned}
U_{DM}(\pi)
&= \sum_{d,y} C^{DM}_{dy} p_\pi(d,y),
\qquad
U_{DS}^s(\pi)
= \sum_{d,y} C^{DS,s}_{dy} p_\pi(d,y \mid s),
\end{aligned}
\end{equation}
where $C^{DM}_{dy}$ and $C^{DS}_{dy}$  encode the cost or benefit associated with decision 
$d$ and outcome $y$ for the DM and DS, respectively. {This affine form is an assumption, commonly adopted in welfare-based formulations of fairness~\cite{heidari2018fairness}, where stakeholder utilities are modeled as expectations over decision-outcome pairs. In particular, starting from an expectation-based utility (Eq.~\ref{eqn:udm}), taking expectation over $(X,S)$ yields a decomposition of the form: \[
P_\pi(D=d,Y=y)\,C_{dy}
= \mathbb{E}_{X,S}\big[\pi(D=d\mid X,S)\,P(Y=y\mid X,S)\,C_{dy}\big],
\] which shows that utilities can be written as linear combinations of the joint probabilities $p_\pi(d,y)$. Thus, marginalizing over $(X,S)$ recovers exactly these joint decision-outcome probabilities. } 
Once a policy fixes how often each $(d,y)$ pair occurs, all stakeholder utilities are determined by linear aggregation. This affine structure is central to our analysis, it implies that mixtures of policies correspond to convex combinations of utility outcomes, which in turn makes the geometry of utility-space Pareto fronts tractable.

\begin{proposition}[Geometry of Utility-Space Pareto Fronts]\label{prop:hull} 
Assume the utilities $U_{DM}(\pi)$ and $U_{DS}(\pi)$ to be affine functions of the joint decision-outcome probabilities,  and the SP utility measuring fairness to be computed as $U_{SP} = \mathcal{R}(U_{DS}(\pi))$, where $\mathcal{R}$ is convex and minimized (or concave and maximized). 
Then, stochastic policies strictly expand the attainable Pareto front if and only
if the deterministic front violates the convex (concave) curvature implied by $\mathcal R$. \TightenPar{1}
\end{proposition} 
\emph{Interpretation.}
Deterministic policies induce a discrete set (possibly non-smooth) of utility outcomes whose PF may have
gaps or incorrect curvature. Stochastic policies generate convex combinations of these
outcomes, filling in
the lower convex (upper concave) envelope of {the PF of deterministic policies ($\mathrm{PF}_{\det}$)}. %
Consequently, stochasticity does not introduce new extreme solutions; they improve trade-offs
only when the deterministic PF violates the curvature implied by the SP’s utility. In the following corollaries, we characterize when such curvature violations arise for two justice-based notions of fairness, Egalitarian and Rawlsian.

\paragraph{Symmetry and Alignment of Utilities.}
Two structural properties determine the geometry of utility-space trade-offs.
DM utilities are \emph{near-symmetric} when the benefits and costs of
positive decisions are comparable, i.e., $|C^{DM}_{11}| \approx|C^{DM}_{10}|$, and
\emph{asymmetric} when $|C^{DM}_{11}|>>|C^{DM}_{10}|$. %
Alignment between the DM and a DS group $s$ is defined as
$\alpha_s=\langle C^{(DM)},C^{(DS)}_s\rangle$, where $\langle \cdot \rangle$ is a inner product.
$\alpha_s>0$ (aligned) implies both stakeholders benefit from the same
decision-outcome events, while $\alpha_s<0$ (misaligned) indicates a conflict of
interests. These two properties jointly determine when deterministic PFs are
well-shaped and when stochastic policies become beneficial.

\begin{corollary}[Egalitarian Social Planner]
\label{cor:egal}
Assume the utilities $U_{DM}(\pi)$ and $U_{DS}(\pi)$ to be affine functions of  the joint decision-outcome probabilities,  and an Egalitarian SP, such that $
    \mathcal R_{\mathrm{Egal}}
    $ is a convex function and minimized. Then, stochastic policies strictly expand the attainable Pareto front if the DM utility is asymmetric ($|C^{DM}_{11}|>>|C^{DM}_{10}|$). Otherwise,  as the DM utility becomes nearly-symmetric, the stochastic and deterministic Pareto fronts coincide. %

\end{corollary}

\begin{corollary}[Rawlsian Social Planner]
\label{cor:rawls}
Assume the utilities $U_{DM}(\pi)$ and $U_{DS}(\pi)$ to be affine functions of  the joint decision-outcome probabilities,  and an Rawlsian SP, such that $
    \mathcal R_{\mathrm{Rawls}}
    $ is a concave function and maximized. Then, stochastic policies strictly expand the attainable Pareto front if the DM utility is asymmetric ($|C^{DM}_{11}|>>|C^{DM}_{10}|$) and misaligned with at least one sensitive group ($\alpha_s<0$). Otherwise, the stochastic and deterministic Pareto fronts coincide.
     \TightenPar{1}
\end{corollary}

\begin{highlightbox}
\textbf{Key Takeaways.} 
Stochastic policies provide the largest gains under Egalitarian fairness, where convex
objectives directly benefit from averaging across deterministic decisions.
Under Rawlsian fairness, improvements arise only in the presence of strong
utility asymmetries combined with misalignment between the DM and DS; otherwise, deterministic and stochastic trade-offs largely overlap. In either case, the foregoing theoretical results can assist practitioners in identifying when stochastic parameterizations may provide fairness improvements relative to deterministic alternatives, and therefore may merit consideration subject to moral acceptability when designing the decision-making policy to be deployed. 
\end{highlightbox}

\begin{wrapfigure}[19]{l}{0.5\linewidth}
  \centering
  \vspace{-1.0\baselineskip}
  \includegraphics[width=\linewidth]{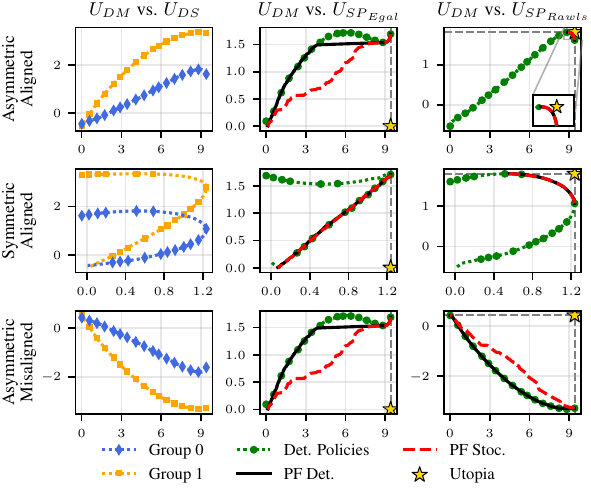}
  \caption{Conditions under which stochastic policies expand the deterministic PF and provide better trade-offs.
  }
  \label{fig:theorem}
  \vspace{-0.5\baselineskip}
\end{wrapfigure}

Figure~\ref{fig:theorem} provides a visual summary of Proposition~\ref{prop:hull} and Corollaries~\ref{cor:egal} and~\ref{cor:rawls}. Each row corresponds to a structural regime defined by the symmetry of DM utilities and their alignment with DS utilities, while columns report the resulting trade-offs under Egalitarian and Rawlsian fairness.
Consistent with the theory, stochastic policies expand the PF precisely when deterministic trade-offs fail to satisfy the curvature implied by the SP’s objective. Under Egalitarian fairness, this occurs whenever there is asymmetry (top and bottom rows), leading to non-convex {deterministic} fronts that are \mbox{convexified} by stochastic policies. Under Rawlsian fairness, deterministic and stochastic fronts largely overlap except in the presence of both strong asymmetry and misalignment (bottom row), where stochasticity smooths non-concave regions and yields additional Pareto-optimal compromises. In near-symmetric regimes (middle row), deterministic fronts are already well-shaped, and stochastic and deterministic trade-offs coincide. Full proofs are provided in Appendix~\ref{app:proofs}.

\section{Performance-fairness trade-offs in practice}
We instantiate our multi-stakeholder (MS) framework by evaluating two classes of decision policies,
$\{\Pi_{\text{det}},\Pi_{\text{stoch}}\}$, that share the same risk score $h$ and differ only in how they map scores to decisions.
$\Pi_{\text{det}}$ consists of deterministic thresholds (parameter $\tau$), while $\Pi_{\text{stoch}}$ consists of score-dependent randomized rules (parameters $(\beta,\gamma)$).
Each class can be instantiated as \emph{shared} (one parameter set for all groups) or \emph{group-specific} (parameters $\tau_s$ or $(\beta_s,\gamma_s)$ for each $s\in\mathcal S$).
Shared policies apply when group membership is unavailable or impermissible at decision time; group-specific policies apply when group membership is available and its use is legally and ethically permissible.

\subsection{Approximating the Pareto Fronts}

\begin{wrapfigure}{r}{0.4\linewidth}
  \centering
   \vspace{-1.0\baselineskip}
  \includegraphics[width=\linewidth]{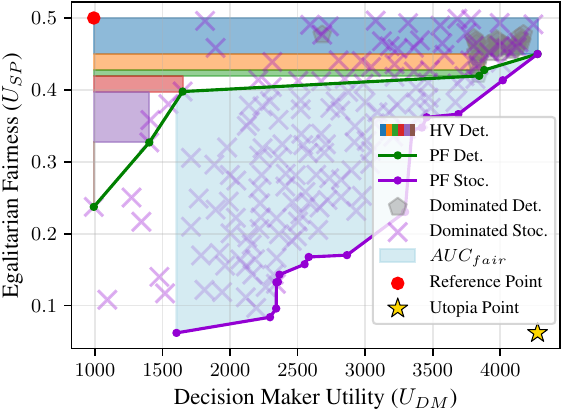}
  \caption{Visual representation for calculation of the HV using step stair method. Each rectangle represents the contribution of a particular Pareto policy in the estimation of HV.}
  \label{fig:PFandHV}
\end{wrapfigure}

Given a finite set of evaluated policies $\{\pi\}\subset\Pi$, we approximate the Pareto front (PF) in the chosen objective space $\mathbf f(\pi)$ by retaining the \emph{non-dominated} policies: those for which no other policy improves one objective without worsening the other (see Algorithm~\ref{alg:get_pareto} in Appendix~\ref{app:algorithms}).
Throughout, we trade off DM utility/performance ($U_{DM}$) against a fairness objective for decision subjects ($U_{SP}$), instantiated as (i) minimizing utility disparity (Egalitarian), or (ii) maximizing the minimum group utility (Rawlsian). In all Pareto plots, the \textbf{utopia point} corresponds to the (hypothetical) policy achieving maximal performance and zero or near zero unfairness, while the \textbf{nadir point} represents the worst observed non-dominated outcome. These points serve as anchors for comparing policy classes.\footnote{In our experiments, the utopia and nadir points were defined by the best and worst utility values on the Pareto front, respectively. Additional details are provided in Appendix~\ref{app:exp_details}.}

\subsection{Comparing Policy Classes in Utility space} \label{sec:metrics}

In a multi-stakeholder setting, comparing policy classes based on a single operating point is insufficient, as it implicitly fixes a particular performance-fairness trade-off. Instead, we compare entire PFs, capturing differences in \emph{coverage, shape, and fairness benefit}. {To this end, we adopt the standard MOO metric hypervolume~\cite{zitzler1999evolutionary,audet2021performance,riquelme2015performance} and related variants, which we adapt to our setting with clear stakeholder interpretations.}\TightenPar{1} 

\begin{itemize}
    \item \textbf{(Normalized) Hypervolume.}  
    For a $\text{PF}_\Pi \subset \mathbb R^2$, the hypervolume $\text{HV}(\Pi)$ is the area dominated by $\text{PF}_\Pi$ relative to reference point $r = (U_{DM}^{\text{r}}, U_{SP}^{\text{r}})$, chosen slightly worse than the nadir point:
    \[
        \text{HV}(\Pi) = \iint_{\text{dominated region}(\text{PF}_\Pi,r)} dU_{DM}\, dU_{SP}.
    \]
    We compute HV via the step-stair decomposition~\cite{zitzler1998multiobjective,while2011fast,shang2020survey} (Algorithm~\ref{alg:hypervolume} in Appendix~\ref{app:algorithms}).
    To compare across datasets, we report $nHV(\Pi)=\text{HV}(\Pi)/HV_{\max}$, where $HV_{max} = (U_{DM}^{\text{utopia}}- U_{DM}^{\text{r}})(U_{SP}^{\text{utopia}}- U_{SP}^{\text{r}})$, is the area of the rectangle spanned by the utopia point and $r$. A visual representation of HV calculation is given in Figure~\ref{fig:PFandHV}.
   
   \item \textbf{Fairness gain.}
   While HV captures global coverage, it does not isolate fairness improvements at comparable performance levels. To this end, we define a fairness gain metric that measures how much one policy class improves fairness relative to another at the same DM utility. {This metric is inspired by hypervolume-based difference measures used in MOO~\cite{bossek2018performance}, but adapted to compare fairness at fixed performance levels.}
    Let $U_{SP}^{\max}(\Pi,u)$ be the best fairness achievable by class $\Pi$ at $U_{DM}=u$ %
    .
    The point-wise gain of policy class $\Pi_b$ over $\Pi_a$ is given by:
    \[
        \delta_f(u)=U_{SP}^{\max}(\Pi_b,u)-U_{SP}^{\max}(\Pi_a,u),
    \]
    here $ \delta_f(u) >0$ indicates that $\Pi_b$ is fairer that $\Pi_a$ at $u$
    and we summarize gains over the shared performance range via the area between envelopes:
    \[
        \text{AUC}_{\text{fair}}=\int_{U_{DM}^{\min}}^{U_{DM}^{\max}} \delta_f(u)\, du.
    \]
\end{itemize}

\paragraph{Interpretation in the Multi-Stakeholder Framework}
$nHV$ measures how much \emph{negotiation space} between performance and fairness a policy class offers, while fairness gain isolates whether one class delivers systematically higher subject welfare at comparable DM utility.
Together, these metrics implement our \emph{first-see-then-design} protocol: stakeholders compare feasible trade-offs before selecting an operating point i.e., a deployable decision policy, making normative choices \emph{explicit and transparent} rather than implicit in a single trained model supporting informed, participatory decision-making among stakeholders.

\section{Experiments}\label{sec:experiments}

\begin{wrapfigure}{r}{0.34\linewidth}
  \centering 
  \vspace{-1.0\baselineskip}
  \includegraphics[width=\linewidth]{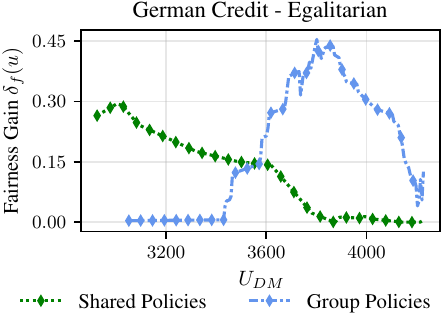}
  \caption{Egalitarian fairness Gain across utility spectrum, for both shared and group-specific for German Credit dataset.}
  \label{fig:auc_gainperutility}
\end{wrapfigure}

\paragraph{Datasets.}
We evaluate on five datasets: two synthetic (credit, hiring) and three real, German Credit~\cite{hofmann1994statlog}, Home Credit~\cite{matthyspredicting,homecredit}, and MIMIC-III Sepsis~\cite{Hou2020} spanning finance, hiring, and healthcare and differing in group imbalance. For each dataset, we specify domain-motivated DM and DS utilities via interpretable cost-benefit constants over decision outcomes, and the SP measures fairness by comparing DS utilities across groups under a chosen justice criterion. Full dataset descriptions, preprocessing, and utility and policy specifications are provided in Appendix~\ref{app:datasets}.

The utility constants are not intended as normative prescriptions; they simply parameterize stakeholder preferences. Our theory does not depend on their specific values but only on: (i) utilities being affine in decision-outcome probabilities and (ii) the SP objective having known curvature. As a result, qualitative performance-fairness trade-offs are driven by structural properties (e.g., symmetry and stakeholder alignment), not by the precise magnitudes of the constants.

\paragraph{Evaluation.}
Each dataset is split into training and test sets. We first train a neural network classifier $h(X,S)$ on the training data to estimate 
$\mathbb{P}(Y=1\mid X,S)$, then freeze it and compute scores for both splits. Using the training scores, we instantiate the policy classes from Section~\ref{sec:methodology} and select Pareto-optimal policies in the utility objective space. We then evaluate these fixed policies on the test set using the performance-fairness metrics in Section~\ref{sec:metrics}. Since training-Pareto optimality need not generalize, we report empirical PFs and metrics computed from test-set utilities.\footnote{The complete implementation is available at \href{https://github.com/kavyagupta/First-See-Then-Design.git}{the official repository}.} \TightenPar{1}
 
\begin{figure*}
    \centering    \includegraphics[width=1\linewidth]{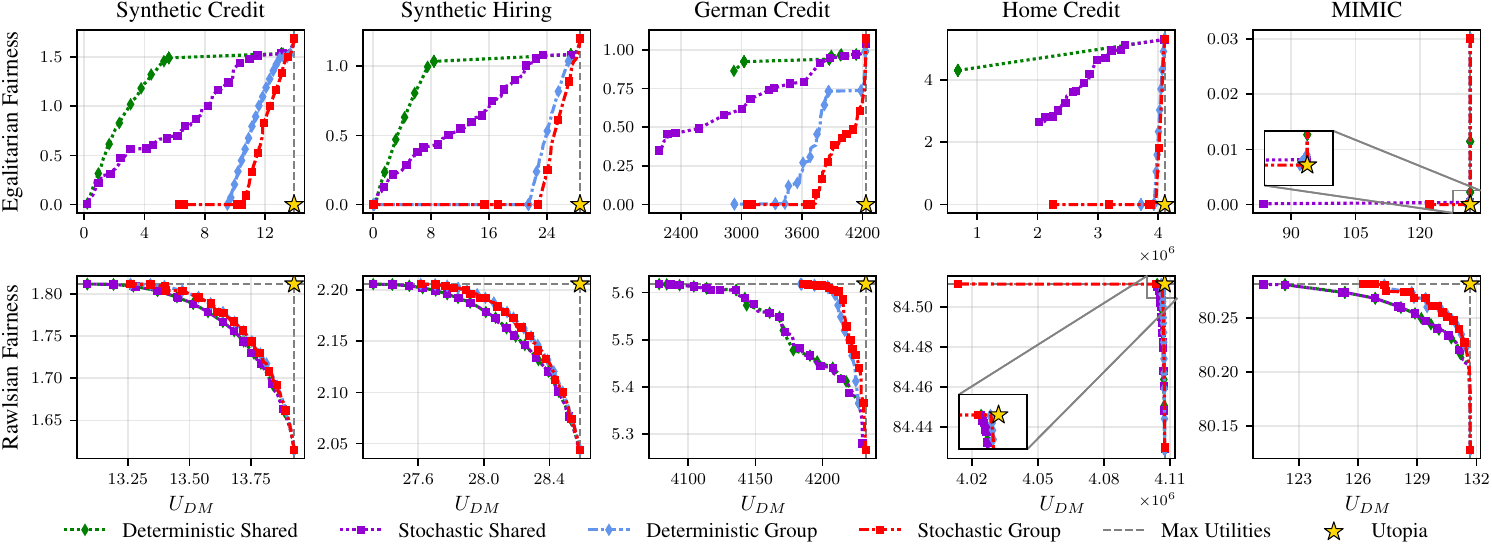}
    \caption{\textbf{PFs comparing deterministic and stochastic policies.} Stochastic policies consistently expand the PF, outperforming their deterministic counterparts in both shared  and group-specific settings. Notably, stochastic group-specific policies trace broader and fairer regions, approaching the utopia point, where both stakeholder utilities are optimal.\TightenPar{1}}
    \label{fig:paretofronts}
\end{figure*}

\begin{table*}[t]
\centering
\scriptsize
\setlength{\tabcolsep}{3.3pt} %
\begin{tabular}{ll|cc|cc|c|c|cc|cc|c|c}
\toprule
\multirow{3}{*}{\textbf{Dataset}} & \multirow{3}{*}{\textbf{Setting}} 
& \multicolumn{6}{c|}{\textbf{Egalitarian}} 
& \multicolumn{6}{c}{\textbf{Rawlsian}} \\
\cmidrule(lr){3-8}\cmidrule(lr){9-14}
& 
& \multicolumn{2}{c|}{${nHV}$~($\uparrow$)} 
& \multicolumn{2}{c|}{${nHV_{test}}$~($\uparrow$)} 
& \multicolumn{1}{c|}{$AUC_{fair}$} 
& \textbf{$AUC_{fair}^{test}$}
& \multicolumn{2}{c|}{${nHV}$~($\uparrow$)} 
& \multicolumn{2}{c|}{${nHV_{test}}$~($\uparrow$)} 
& \multicolumn{1}{c|}{$AUC_{fair}$} 
& \textbf{$AUC_{fair}^{test}$} \\
&
& \textbf{$PF_{det}$} & \textbf{$PF_{stoch}$} 
& \textbf{$PF_{det}$} & \textbf{$PF_{stoch}$} 
& \multicolumn{1}{c|}{} & 
& \textbf{$PF_{det}$} & \textbf{$PF_{stoch}$} 
& \textbf{$PF_{det}$} & \textbf{$PF_{stoch}$} 
& \multicolumn{1}{c|}{} & \\
\midrule
\multirow{2}{*}{Synthetic Credit} & shared & 0.26& 0.47& 0.26& 0.48 & 0.36& 0.38 & 0.82& 0.83& 0.84& 0.84 & 0.00& 0.00\\
     & group & 0.82& 0.87& 0.81& 0.87 & 0.27& 0.29 & 0.86& 0.86& 0.86& 0.87 & 0.00& 0.00\\
    \midrule
    \multirow{2}{*}{Synthetic Hiring} & shared & 0.24& 0.47& 0.25& 0.47 & 0.26& 0.24 & 0.81& 0.81& 0.82& 0.82 & 0.00& 0.00\\
     & group & 0.87& 0.90& 0.88& 0.91 & 0.04& 0.16 & 0.85& 0.86& 0.85& 0.86 & 0.00& 0.00\\
    \midrule
    \multirow{2}{*}{German} & shared & 0.18& 0.35& 0.29& 0.31 & 0.12& 0.00 & 0.76& 0.77& 0.80& 0.80 & 0.00& 0.00\\
     & group & 0.82& 0.91& 0.63& 0.96 & 0.18& 0.12 & 0.93& 0.95& 0.98& 0.97 & 0.01& 0.00\\
    \midrule
    \multirow{2}{*}{HomeCredit} & shared & 0.02& 0.31& 0.01& 0.30 & 0.68& 0.40 & 0.99& 0.99& 0.85& 0.85 & 0.00& 0.00\\
     & group & 0.97& 0.98& 0.94& 0.98 & 0.23& 0.62 & 1.00& 1.00& 0.94& 0.95 & 0.00& 0.00\\
    \midrule
    \multirow{2}{*}{MIMIC} & shared & 0.97& 1.00& 0.16& 0.16 & 0.00& 0.00 & 0.89& 0.89& 0.89& 0.91 & 0.00& 0.00\\
     & group & 1.00& 1.00& 0.98& 0.99 & 0.00& 0.04 & 0.95& 0.95& 0.95& 0.95 & 0.00& 0.00\\
\bottomrule
\end{tabular}
\caption{{nHV and  $AUC_\text{fair}$ for different policy classes in utility space.} Higher nHV values indicate broader Pareto-optimal trade-off regions. Stochastic policies consistently show improvement in nHV values. Similarly high AUC values indicate better fairness gain of stochastic policies over deterministic policies. Unnormalized HV values can be found in Table~\ref{tab:app_hv_auc_results}.}
\label{tab:hv_auc_results}
\vspace{-0.7cm}
\end{table*}

\paragraph{Comparison of policies}
Figure~\ref{fig:paretofronts} shows train time utility-space Pareto fronts under Egalitarian and Rawlsian justice. Consistent with our theory in Section~\ref{sec:theory}, stochastic policies often strictly expand the deterministic front under Egalitarian fairness, whereas under Rawlsian fairness the two fronts largely overlap, with stochasticity yielding modest but consistent gains in some regimes.

Across all datasets (Figure~\ref{fig:paretofronts}), \textbf{stochastic policies improve fairness without substantial loss in DM utility}. This pattern is corroborated quantitatively in Table~\ref{tab:hv_auc_results}: stochastic policies yield higher normalized HV, indicating a broader and more favorable trade-offs. A larger HV reflects a more flexible front, enabling smoother transitions between performance and fairness an important property for real-world deployment.
Stochastic policies also attain positive integrated fairness gains ($AUC_{fair}$) across all datasets, confirming that stochasticity improves fairness not only at isolated operating points but across the utility spectrum (Figure~\ref{fig:auc_gainperutility}).

Group-specific policies outperform shared ones across both deterministic and stochastic classes. By tailoring decision parameters to group-conditional score distributions and utilities, they attain higher HV and larger fairness gains. Shared policies are simpler and align with fairness-through-unawareness, but are less expressive and may still produce disparate impact via proxy features; group-specific policies can reduce such disparities, at the cost of legal and ethical concerns about differential treatment. Our results quantify this tension between deployability and the benefits of group-aware optimization. Notably in Figure~\ref{fig:auc_gainperutility}, shared stochastic policies consistently beat deterministic baselines over a broad utility range, with the largest gains in intermediate regimes where performance-fairness conflicts are sharpest. Group-specific stochastic policies deliver the strongest fairness improvements in high-utility regions, which are particularly relevant in high-stakes settings where institutional objectives must be maintained. Our empirical results demonstrate that group-specific policies can achieve substantial improvements, highlighting that the magnitude of these gains may warrant a careful, principled discussion on the collection and use of sensitive attributes for policy optimization.  \TightenPar{1}
\begin{wrapfigure}{l}{0.365\linewidth}
  \centering
  \includegraphics[width=0.81\linewidth]{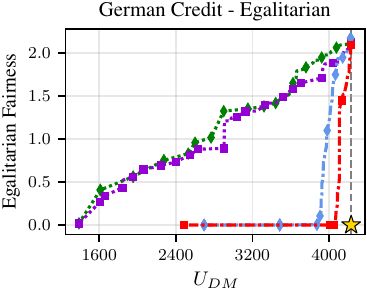}
  \medskip
  \centering
  \vspace{5pt}
  \begin{tabular}{c|cc|c}
    \textbf{Setting} & \textbf{$nHV_{det}$} & \textbf{$nHV_{stoc}$} & \textbf{$AUC_{fair}$} \\
    \midrule
    shared & 0.47 & 0.52 & 0.08 \\
    group  & 0.92 & 0.97 & 0.19 \\
  \end{tabular}

  \caption{PFs, nHV and $AUC_\text{fair}$ when each DS group has different utility constants.}
\label{fig:diff_constants}
\end{wrapfigure} 

For completeness, we report German Credit results with \emph{heterogeneous} DS utilities (i.e., female applicants have a higher benefit when a loan is approved and repaid). Figure~\ref{fig:diff_constants} shows the same qualitative patterns as the main experiments, confirming that our empirical findings are \textbf{robust to heterogeneity in group-level utility specifications}. 
{Finally, in the synthetic credit setting, we perform an ablation on the utility constants to study how hypervolume gains depend on utility asymmetry and alignment (Section~\ref{sec:theory}). As illustrated in Figure~\ref{fig:HVgainVSratio}, increasing asymmetry ($|C_{11}^{DM}| / |C_{10}^{DM}| > 1$) leads to larger gains when switching to stochastic policies, with the effect being particularly pronounced under misalignment. In contrast, under alignment, gains remain negligible for the Rawlsian objective across most ratios. Notably, Egalitarian gains are invariant to alignment, and no improvements are observed when $|C_{11}^{DM}| \leq |C_{10}^{DM}|$.
These observations are consistent with our theoretical results: asymmetry alone is sufficient to induce gains under the Egalitarian objective (Corollary~\ref{cor:egal}), whereas Rawlsian improvements additionally require misalignment (Corollary~\ref{cor:rawls}). Full results and additional analyses are provided in Appendix~\ref{app:stochasticity}.}

Additional results (including test-set PFs, per-utility fairness gains and unnormalized HVs) appear in Appendix~\ref{app:extra_results}. Overall, stochastic policies especially shared variants provide a robust and practical way to improve performance-fairness trade-offs without materially compromising the DM objectives. \TightenPar{1}

\begin{figure*}
    \centering
    \includegraphics[width=1\linewidth]{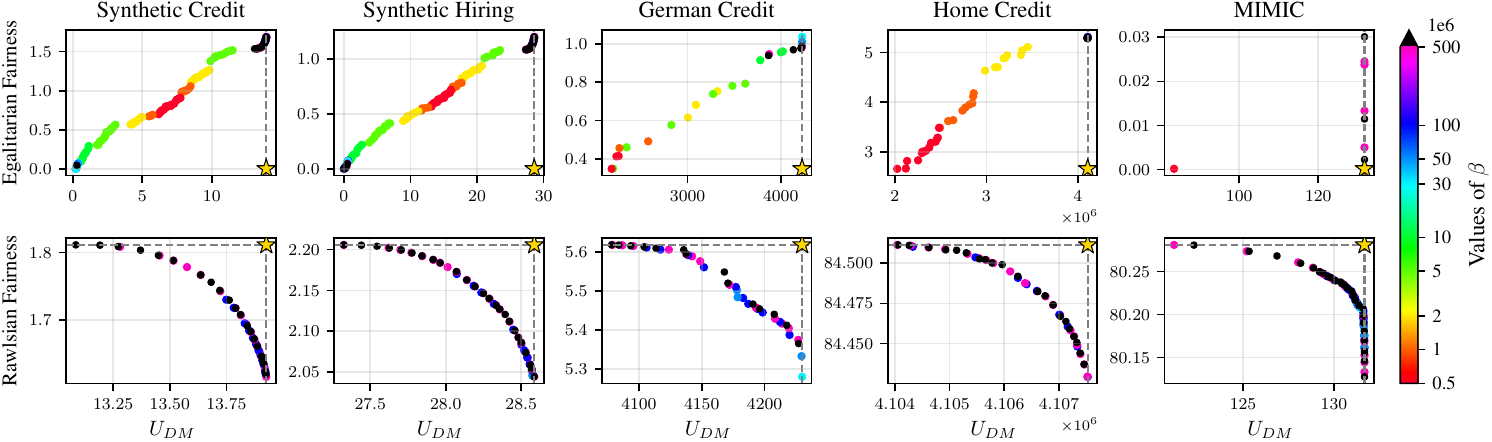}
    \caption{\textbf{Effect of stochasticity.} Shared stochastic policies (colored by $\beta$) populate interior regions of the utility-space PFs, enabling smoother and more flexible trade-offs than deterministic policies.}
    \label{fig:beta_heatmap_main}
    \vspace{-0.3cm}
\end{figure*}

\paragraph{Effect of stochasticity.}
Figure~\ref{fig:beta_heatmap_main} shows how the stochasticity parameter 
$\beta$ shapes the utility-space PFs for shared stochastic policies. Each point is a Pareto-optimal policy, color-coded by 
$\beta$: small 
$\beta$ (warmer colors) yields highly stochastic decisions, while large 
$\beta$ (cooler colors) approaches deterministic thresholding. Across datasets, intermediate 
$\beta$ values populate the interior of the PF, filling gaps between deterministic solutions. This matches the geometric interpretation in Section~\ref{sec:theory}: by interpolating between deterministic outcomes, stochastic policies convexify (and, under Rawlsian objectives, effectively concavify along the relevant direction) the attainable utility region, \textbf{enabling gradual fairness improvements without abrupt losses in DM utility}. %
\begin{wrapfigure}{r}{0.45\linewidth}
    \centering
\includegraphics[width=\linewidth]{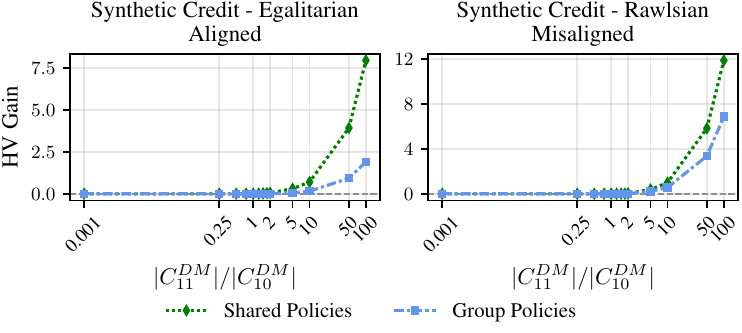}
    \caption{Hypervolume gain as a function of the utility asymmetry ratio $|C_{11}^{DM}| / |C_{10}^{DM}|$ for both alignment and misalignment cases on the synthetic credit dataset.}

    \label{fig:HVgainVSratio}
\end{wrapfigure}
Figure~\ref{fig:pY_vs_Pz} provides a policy-level view for group-specific stochastic policies on German Credit. It plots the loan-granting probability 
$\pi_d$ as a function of predicted repayment probability $P_y$, for each group. \textbf{Stochasticity is applied selectively}: one group follows an almost deterministic rule (large 
$\beta$), while the other uses a softer boundary (moderate 
$\beta$). This asymmetric use of stochasticity allows the policy to target fairness constraints e.g., meeting an Egalitarian fairness level without unnecessarily randomizing decisions for all groups. More broadly, stochasticity concentrates where it is most useful i.e. near decision boundaries. %
Appendix~\ref{app:stochasticity} further clarifies these dynamics. When plotting \emph{all} evaluated shared stochastic policies (not only PF points; Figure~\ref{fig:app_betaheatmaps}), policies cluster by $\beta$, and these clusters shrink as $\beta$ decreases consistent with increased smoothing making the exact threshold less consequential.

\begin{highlightbox}
\textbf{Key Takeaways.} Stochastic policies strictly generalize deterministic ones, so our sweep still recovers optimal near-deterministic solutions when added randomness is unhelpful. Allowing group-specific parameters further expands the attainable trade-off set. Finally, our multi-stakeholder framework places DM, DS, and SP objectives in a common utility space, enabling transparent, collaborative selection of decision policies that reflect stakeholder priorities. \TightenPar{1}
\end{highlightbox}

\vspace{10pt}
\section{Discussion}
\paragraph{Fairness as a negotiated trade-off, not a constraint:}
A central claim of this work is that fairness in algorithmic decision-making should be understood not as a rigid constraint on performance, but as a negotiated trade-off grounded in realized welfare~\cite{heidari2019moral}. Classical fairness notions operate in the predictive space and rely on error-based proxies that are often misaligned with social outcomes~\cite{corbett2023measure,scantamburlo2025prediction}. We show that enforcing such constraints can yield policies that are suboptimal when evaluated in terms of stakeholder utilities~\cite{hertweck2024group}. By shifting analysis to the utility space and adopting a \textbf{post-hoc multi-objective optimization approach}, we make performance-fairness trade-offs explicit and interpretable. Each Pareto-optimal policy represents a legitimate normative choice balancing DM utility, DS welfare, and social fairness, supporting \emph{ex ante deliberation} rather than prescribing a single “fair” solution.

\paragraph{Blind vs. Non-blind Policies; Ethical and Legal Implications:}
Our results show that policy parameterization shapes attainable trade-offs. When permissible, group-specific policies can improve performance-fairness outcomes by accounting for distributional differences and heterogeneous impacts~\cite{hertweck2024group}, but using sensitive attributes raises legal and ethical concerns~\cite{barocas2016big}. Notably, \textbf{shared (group-blind) stochastic policies can still yield meaningful fairness gains}~\cite{awasthi2020equalized}, offering a pragmatic middle ground between compliance and welfare improvement. They also align with an individual-fairness view by mapping similar individuals to similar decision probabilities and avoiding sharp threshold discontinuities~\cite{dwork2012fairness}. Blind and non-blind policies should thus lie on a design spectrum rather than as a binary choice.

\paragraph{Expanding the Fairness Landscape with Stochastic Policies:}
A key technical contribution is demonstrating how \textbf{{stochastic} decision policies expand the set of attainable performance-fairness trade-offs}. As a principled modeling choice rather than ad-hoc randomization~\cite{agarwal2018reductions}, stochasticity smooths threshold effects and mitigates welfare discontinuities near decision boundaries. Stochastic policies generalize deterministic ones, preserving their trade-offs while filling gaps in the utility-space PF under misaligned objectives, so their benefits are setting-dependent. {Importantly, our goal is not to advocate for unconditional deployment of stochastic policies, but to provide tools for identifying \emph{when} and \emph{by how much} stochasticity can improve fairness relative to deterministic policies. This allows stakeholders to assess whether such improvements justify the use of randomization in a given context. In practice, stochasticity is already used in institutional settings (e.g., lotteries or tie-breaking).} At the same time, stochasticity raises procedural and governance considerations, requiring clear communication, auditability, and justification for when randomization is applied. These concerns motivate confidence-aware or semi-stochastic variants (Figure~\ref{fig:pY_vs_Pz}) that restrict randomization to high-uncertainty regions.\TightenPar{1}

\begin{wrapfigure}[15]{l}{0.4\linewidth}
  \centering
  \vspace{-5pt}
\includegraphics[width=0.9\linewidth]{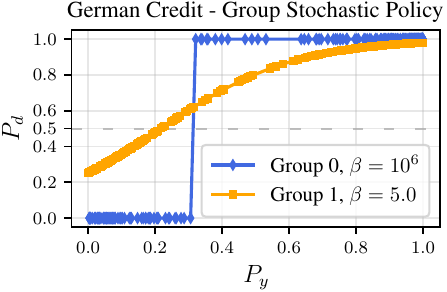}
  \caption{Group-specific stochastic policies apply different levels of randomness across demographic groups.  
  The policy is selected at
  $U_{DM} \approx 4000$, $U_{{SP}_{Egal}} \approx 0.5$.}
  \label{fig:pY_vs_Pz}
\end{wrapfigure}

\paragraph{Practicality and real-world deployment:}
Our framework is designed for practical deployment. The proposed stochastic policies require only lightweight reparameterizations of standard decision rules and introduce few, interpretable parameters that can be tuned via utility-space Pareto fronts. This enables fairness improvements without sacrificing scalability or tractability. For regulators and auditors, the framework offers \textbf{transparency grounded in realized welfare} rather than predictive parity, supporting clearer accountability, auditability and oversight.

\paragraph{Limitations and Future Directions:}
The framework relies on specifying stakeholder utilities, which may be contested or imperfectly estimated, underscoring the need for participatory design~\cite{vineis2025beyond} and sensitivity analysis~\cite{selbst2019fairness}. {Our current formulation adopts group-based utilities; extending this to alternative fairness notions, such as distance-based individual fairness, is a promising direction for future work. Our theoretical results are not tied to group-based fairness, but require only a well-defined SP objective with appropriate curvature properties; once specified, our geometric characterization continues to apply.} Our analysis assumes reasonably calibrated outcome estimates and focuses on static binary decisions. {Extending the framework to multi-class decisions or outcomes is conceptually straightforward, provided both the score function and utilities are appropriately defined (e.g., via temperature-scaled softmax policies~\cite{hinton2015distilling}), though extending the theoretical analysis is left to future work.}  Replacing our current policy search with dedicated multi-objective optimization could further improve efficiency in exploring the utility-space Pareto front. %
Finally, while stochastic and group-specific policies can improve trade-offs, deployment may be constrained by legal or institutional factors beyond the scope of this work, highlighting the need for \textbf{integrated technical, legal, and governance approaches}.

\begin{acks} %
{ This work has been supported by the project “Society-Aware
Machine Learning: The paradigm shift demanded by society to trust machine learning,” funded
by the European Union and led by IV (ERC-2021-STG, SAML, 101040177);  and the the Deutsche Forschungsgemeinschaft (DFG, German Research Foundation) through GRK 2853/1 “Neuroexplicit Models of Language, Vision, and Action” (Project No. 471607914). Views and opinions expressed are, however, those of the author(s) only and do not necessarily reflect
those of the aforementioned funding agencies. Neither of the aforementioned parties can be held
responsible for them.}
\end{acks}

\section*{Generative AI usage statement }

This research \textbf{did not} use LLMs/Generative AI in any part of the core methodology (problem formulation, algorithm design, implementation, experiments, or evaluation). If LLM tools were used at all, it was only for non-substantive writing or editing, grammar, word choice and \textbf{did not} influence results, claims, or conclusions presented in the work.

\bibliographystyle{ACM-Reference-Format}
\bibliography{facct26}

\clearpage
\appendix

\section{Related Work} \label{app:related_work}

\paragraph{Multi-Stakeholder Fairness and Multi-objective Optimization}
Traditional fairness approaches often treat fairness as a constraint on a single loss function typically optimized from the decision-maker’s perspective. In reality, algorithmic decisions affect multiple stakeholders with conflicting goals, such as institutional utility, individual well-being, and regulatory oversight. Recent work re-frames fairness as a separate objective, modeling each stakeholder via their own utility function~\cite{heidari2019moral, liu2022accuracy, baumann2022distributive}.
MOO provides a principled framework for balancing these competing goals. Early work introduced fairness metrics like demographic parity and equalized odds as objectives within a Pareto optimization setup~\cite{zafar2017fairness, donini2018empirical}. Later extensions enabled more expressive models to approximate full Pareto fronts, allowing stakeholders to select trade-offs post-training~\cite{martinez2020minimax}.
However, performance-fairness trade-offs are highly context dependent. Studies have shown that enforcing fairness constraints can reduce accuracy or institutional utility~\cite{corbett2017algorithmic, chouldechova2017fair}. As~\citet{kearns2019ethical} argue, resolving these trade-offs requires making the underlying values explicit clarifying what fairness means and whose interests should be prioritized. Our work builds on this foundation by modeling stakeholder utilities directly and uncover Pareto-efficient solutions, especially in settings where deterministic policies are too rigid. 

\paragraph{Utility-Based Views of Fairness.}
An emerging perspective frames fairness in ML systems as a problem of distributive justice: how should the benefits and harms of algorithmic decisions be allocated? This utility-based view emphasizes that individuals experience different levels of benefit or harm from the same decision, requiring fairness to be evaluated in terms of real-world impact on well-being~\cite{sen1980equality, lundgard2020measuring}.
\citet{heidari2018fairness} introduced welfare-based definitions that integrate downstream consequences into learning objectives. Building on this, \citet{heidari2019moral} argue that fairness must reflect morally relevant claims to utility, inspired by the theory of equality of opportunity. In this view, fairness metrics encode implicit normative judgments about which inequalities are justified.
\citet{hertweck2023justice, hertweck2021moral} extend this line by proposing a framework linking fairness modeling choices to principles like Rawlsian max-min and prioritarianism. They frame algorithms as institutional mechanisms subject to moral scrutiny. It has also been empirically shown that applying fairness constraints without considering utility can harm marginalized groups~\cite{hu2020fair}.
Most recently \citet{casacuberta2023augmenting} propose a unified framework that evaluates fairness through both equity and overall welfare. Together, these works motivate a shift from parity-based criteria to fairness definitions grounded in utility, justice, and stakeholder well-being.

\paragraph{Stochastic Policies and Fairness in Expectation.}
Deterministic decision rules are often insufficient to satisfy fairness constraints, especially when fairness is defined over group-level distributions or involves risk trade-offs across heterogeneous populations. Prior work has shown the necessity of stochasticity in various settings: \citet{wang2021stochastic} demonstrate that randomized policies are essential for exposure fairness in online allocation; \cite{wei2021decision, rateike2022don, kilbertus2020fair} show that deterministic rules in selective label settings reinforce historical bias and limit exploration; and~\cite{celis2019classification} propose randomized classifiers as a principled solution when group fairness constraints are otherwise infeasible.
In bandit learning it has been argued that stochastic exploration is necessary to satisfy fairness objectives like demographic or calibrated parity~\cite{joseph2016fairness, liu2017calibrated}. At the individual level~\citet{dwork2012fairness} formalize fairness as requiring similar individuals to receive similar distributions over outcomes again necessitating randomization.
Building on these insights, our work extends the case for stochasticity to static decision environments. We show that performance-fairness trade-offs often cannot be achieved through deterministic policies alone. Stochastic decision-making becomes essential when stakeholder utilities conflict or group-specific cost asymmetries exist, allowing fairness constraints to be satisfied in expectation and enabling access to broader Pareto-optimal regions.
\section{Proofs}
\label{app:proofs}

We begin by formalizing how stochasticity affects performance and fairness
utilities, and how this translates into the geometry of attainable trade-offs.

\subsection*{Preliminaries}

As a reminder, for any policy $\pi$, the DM utility and DS utilities are given by:
\begin{align}
U_{DM}(\pi)
&= \sum_{d,y} C^{DM}_{dy}\,\mathbb{P}_{\pi}(D=d,Y=y), \label{eq:udm_affine}\\
U_{DS}^s(\pi)
&= \sum_{d,y} C^{DS,s}_{dy}\,\mathbb{P}_{\pi}(D=d,Y=y\mid S=s). \label{eq:uds_affine}
\end{align}
Hence, both $U_{DM}$ and $U_{DS}^s$ are affine functions of the
decision-outcome distribution induced policy by $\pi$.
The SP evaluates fairness as:
$
U_{SP}(\pi) = \mathcal R\!\left(U_{DS}(\pi)\right),
$
where $\mathcal R$ is either convex and minimized (Egalitarian fairness) or concave and
maximized (Rawlsian fairness).

\subsection*{Key Lemmas}
\begin{lemma}[Affine Performance, Curved Fairness]
\label{lem:affine_curved}
For any two deterministic policies $\pi_1,\pi_2$ and any $\lambda\in[0,1]$, let
$\pi_\lambda=\lambda\pi_1+(1-\lambda)\pi_2$ be their stochastic policy that randomizes between them.
Then:
\[
U_{DM}(\pi_\lambda)
=
\lambda U_{DM}(\pi_1)+(1-\lambda)U_{DM}(\pi_2),
\]
while

\[
U_{SP}(\pi_\lambda)
\begin{cases}
\le \lambda U_{SP}(\pi_1)+(1-\lambda)U_{SP}(\pi_2),
& \begin{aligned}[t]
  &\text{if $\mathcal R$ is convex
  and minimized,}
  \end{aligned}
\\[6pt]
\ge \lambda U_{SP}(\pi_1)+(1-\lambda)U_{SP}(\pi_2),
& \begin{aligned}[t]
  &\text{if $\mathcal R$ is concave and maximized.}
  \end{aligned}
\end{cases}
\]

\end{lemma}

\begin{proof}\renewcommand{\qedsymbol}{}
We define the stochastic mixture $\pi_\lambda$ as follows.
Before observing $(X,S)$, a Bernoulli random variable
$Z\sim\mathrm{Bernoulli}(\lambda)$ is sampled.
If $Z=1$, decisions are generated according to $\pi_1$; if $Z=0$, decisions are
generated according to $\pi_2$.

Since this randomization occurs prior to outcome realization, the induced joint
distribution of decision–outcome pairs satisfies:
\[
\mathbb{P}_{\pi_\lambda}(D=d,Y=y)
=
\lambda\,\mathbb{P}_{\pi_1}(D=d,Y=y)
+
(1-\lambda)\,\mathbb{P}_{\pi_2}(D=d,Y=y),
\]
and analogously for conditional distributions given $S=s$.

By assumption, the decision-maker and decision subject utilities are affine in the joint decision-outcome
distribution and substituting this relation into the equations (\ref{eq:udm_affine}) and (\ref{eq:uds_affine})
we obtain:
\[
U_{DM}(\pi_\lambda)
=
\lambda U_{DM}(\pi_1)+(1-\lambda)U_{DM}(\pi_2),
\]
and
\[
U_{DS}(\pi_\lambda)
=
\lambda U_{DS}(\pi_1)+(1-\lambda)U_{DS}(\pi_2),
\]

The SP evaluates fairness as:
\[
U_{SP}(\pi)=\mathcal R\!\left(U_{DS}(\pi)\right).
\]
If $\mathcal R$ is convex, Jensen’s inequality yields:

\[
\begin{aligned}
U_{SP}(\pi_\lambda)
&=
\mathcal R\!\left(\lambda U_{DS}(\pi_1)+(1-\lambda)U_{DS}(\pi_2)\right) \\
&\le
\lambda \mathcal R(U_{DS}(\pi_1))+(1-\lambda)\mathcal R(U_{DS}(\pi_2)).
\end{aligned}
\]

and the inequality reverses if $\mathcal R$ is concave.
Since the social planner minimizes (resp. maximizes) $\mathcal R$, the stated
result follows.
\end{proof}

\begin{lemma}[Curvature of Trade-offs]
\label{lem:curvature}
Let $\tau \mapsto (x(\tau),y(\tau))$ be a continuously parameterized curve with
$x(\tau)$ strictly monotone. The curve is convex if 
$\frac{d y}{d x}$ is non-increasing in $\tau$, and concave if 
$\frac{d y}{d x}$ is non-decreasing. If this monotonicity fails, the curve violates
the corresponding curvature condition.

For deterministic policies parameterized by $\tau$ (decision threshold):
\[
x(\tau) := U_{DM}(\pi_\tau), 
\qquad 
y(\tau) := U_{SP}(\pi_\tau),
\]
so that the $PF_{det}$ is traced by the parametric curve
$\tau \mapsto (x(\tau), y(\tau))$ in utility space.
\end{lemma}

\begin{proof}\renewcommand{\qedsymbol}{}
Because $x'(\tau)\neq 0$, the curve can locally be written as the graph of a
function $y=y(x)$.
By the chain rule:
\[
\frac{d y}{d x}
=
\frac{y'(\tau)}{x'(\tau)}.
\]

The second derivative of $y$ with respect to $x$ is given by:
\[
\frac{d^2 y}{d x^2}
=
\frac{d}{d\tau}\!\left(\frac{d y}{d x}\right)\Big/ x'(\tau)
=
\frac{1}{x'(\tau)}\,\frac{d}{d\tau}\!\left(\frac{y'(\tau)}{x'(\tau)}\right).
\]

If $\frac{d y}{d x}$ is non-increasing in $\tau$, then
$\frac{d}{d\tau}(\frac{d y}{d x})\le 0$.
Since $x'(\tau)$ has constant sign, this implies
$\frac{d^2 y}{d x^2}\le 0$, and hence the curve is convex.

Conversely, if the curve is convex, then $\frac{d^2 y}{d x^2}\le 0$, which implies
$\frac{d}{d\tau}(\frac{d y}{d x})\le 0$, so $\frac{d y}{d x}$ is non-increasing.

The argument for concavity is identical with the inequalities reversed.
\end{proof}

\subsection*{Proof of Proposition~\ref{prop:hull}}

\begin{proof}[Proof Proposition~\ref{prop:hull}]\renewcommand{\qedsymbol}{}

By Lemma~\ref{lem:affine_curved}, stochastic policies generate convex combinations
of deterministic performance values $U_{DM}$, while the fairness coordinate
$U_{SP}$ is improved  according to the curvature of $\mathcal R$.

If $\mathcal R$ is convex and minimized, every stochastic mixture of deterministic
policies yields a point lying weakly below the line segment connecting their
utility pairs. Thus the set of attainable Pareto-optimal points coincides with the
lower convex hull of the $PF_{det}$.
If $\mathcal R$ is concave and maximized, the same argument yields the upper concave
hull.

Conversely, every point on the corresponding hull can be realized by an
appropriate stochastic mixture of deterministic policies. Hence,
\[
\mathrm{PF}_{stoch} = \mathrm{Hull}(\mathrm{PF}_{det}),
\]
where the hull is lower convex or upper concave depending on $\mathcal R$.
Stochastic policies strictly expand the Pareto front if and only if the
deterministic front violates the curvature implied by $\mathcal R$.
\end{proof}

\subsection*{Proof of Corollary~\ref{cor:egal} (Egalitarian)}

\begin{proof}\renewcommand{\qedsymbol}{}
Under Egalitarian fairness, $\mathcal R$ is convex and minimized. By
Proposition~\ref{prop:hull}, stochastic policies expand the Pareto front if and
only if the $PF_{det}$ is non-convex.

Deterministic policies are parameterized by a scalar $\tau$, inducing a trade-off
curve $\tau\mapsto(U_{DM}(\pi_\tau),U_{SP}(\pi_\tau))$. Strong asymmetry in DM utilities
($|C^{DM}_{11}|\gg|C^{DM}_{10}|$) implies that small changes in $\tau$ produce
disproportionately large changes in $U_{DM}$ relative to $U_{DS}$, causing the
marginal rate $\frac{dU_{SP}}{dU_{DM}}$ to vary non-monotonically.
By Lemma~\ref{lem:curvature}, this violates convexity.
As DM utilities become near-symmetric, the marginal trade-offs vary smoothly and
monotonically, yielding an approximately convex $PF_{det}$. In this regime, the lower convex hull coincides with $\mathrm{PF}_{det}$ and stochastic
policies offer no additional Pareto-optimal points.
\end{proof}

\subsection*{Proof of Corollary~\ref{cor:rawls} (Rawlsian)}

\begin{proof}\renewcommand{\qedsymbol}{}
Under Rawlsian fairness, $\mathcal R(U_{DS})=\min_s U_{DS}^s$ is concave and
maximized. By Proposition~\ref{prop:hull}, stochastic policies expand the Pareto
front if and only if the $PF_{det}$ violates concavity.

When DM utilities are near-symmetric, variations in $\tau$ induce smooth,
monotone changes in both $U_{DM}$ and the worst-off group utility, ensuring that
$\frac{dU_{SP}}{dU_{DM}}$ is non-decreasing and the $PF_{det}$ is
approximately concave.
If DM utilities are strongly asymmetric and misaligned with at least one group,
increases in $U_{DM}$ may sharply reduce $\min_s U_{DS}^s$, causing local decreases
in $\frac{dU_{SP}}{dU_{DM}}$. By Lemma~\ref{lem:curvature}, concavity is violated,
and the upper concave hull strictly improves upon the $PF_{det}$.
These improvements are attainable via stochastic policies by
Proposition~\ref{prop:hull}.
\end{proof}

\subsection*{Group-Specific Policies.}
The preceding analysis extends naturally to \emph{group-specific} policies. In this case, the policy class decomposes across groups, and each group $s$ induces its own decision-outcome distribution and utility pair $(U_{DM}^s, U_{DS}^s)$.
As a result, the attainable utility region under group-specific policies is a superset of that attainable under shared policies. \TightenPar{1}

From a \emph{geometric perspective}, allowing group-specific parameters relaxes the coupling between groups imposed by shared thresholds, thereby enlarging the feasible set of joint decision-outcome distributions.
This relaxation can mitigate curvature violations in the $PF_{det}$ by enabling group-level adjustments that align more closely with stakeholder utilities.
Consequently, group-specific deterministic policies may already yield well-shaped (convex or concave) Pareto fronts in regimes where shared deterministic policies do not.

Importantly, Proposition~\ref{prop:hull} continues to apply within the group-specific policy class: stochastic group-specific policies generate convex combinations of group-specific deterministic outcomes and therefore recover the corresponding hull of the deterministic front.
Empirically, this implies that while stochasticity remains beneficial in the presence of curvature violations, the marginal gains from stochasticity are often smaller for group-specific policies, as some geometric limitations are already alleviated by conditioning on group membership.

\section{Pseudo Algorithms}
\label{app:algorithms}

\begin{figure}[H]
\centering

\begin{minipage}[t]{0.485\textwidth}
\captionof{algorithm}{\texttt{get\_pareto\_front}}
\label{alg:get_pareto}
\small
\textbf{Input}: Set of points $P=\{(x,y)\}$ where $x$ = performance (maximize), $y$ = unfairness gap (minimize)\\
\textbf{Output}: Pareto front $P^*$
\begin{algorithmic}[1]
    \State Sort $P$ by \textbf{performance descending} and then by \textbf{unfairness ascending}:
    \[
        P_{\text{sorted}} \gets \text{sort}(P, \text{key}=(-x, y))
    \]
    \State Initialize Pareto front list: $P^* \gets [\ ]$
    \State Initialize best fairness seen: $\texttt{min\_fair} \gets +\infty$
    \For{\textbf{each} $(x, y)$ in $P_{\text{sorted}}$}
        \If{$y < \texttt{min\_fair}$}
            \State Append $(x, y)$ to $P^*$
            \State $\texttt{min\_fair} \gets y$
        \EndIf
    \EndFor
    \State \textbf{return} $P^*$
   
\end{algorithmic}

\end{minipage}
\hfill
\begin{minipage}[t]{0.485\textwidth}
\captionof{algorithm}{\texttt{compute\_hypervolume}}
\label{alg:hypervolume}
\small
\textbf{Input}: Sorted Pareto points $P=\{(x_i,y_i)\}$, reference point $r=(r_x,r_y)$\\
\textbf{Output}: Hypervolume $hv(P,r)$
\begin{algorithmic}[1]
    \State $hv \gets 0.0$
    \For{$i \gets 0$ \textbf{to} $|P|-1$}
        \State $(x_i, y_i) \gets P[i]$
        \If{$i > 0$}
            \State $y_{\text{next}} \gets P[i-1].y$
        \Else
            \State $y_{\text{next}} \gets r_y$
        \EndIf
        \State $width \gets x_i - r_x$
        \State $height \gets y_{\text{next}} - y_i$
        \State $area \gets width \times height$
        \State $hv \gets hv + area$
    \EndFor
    \State \textbf{return} $hv$
\end{algorithmic}
\end{minipage}

\end{figure}

\section{Implementation Details} \label{app:datasets}
\subsection{Experimental Details}\label{app:exp_details}

In our experiments, we explored policy configurations using a comprehensive, brute-force sweep over the relevant hyperparameters. For all policies, we evaluated 100 threshold values denoted by $\gamma$ for stochastic policies and $\tau$ for deterministic policies uniformly spaced in the interval $(0.01, 0.99)$ using a linear grid (i.e., a \texttt{linspace}).
For stochastic policies, we additionally varied the stochasticity parameter $\beta$ over the set: $\{1000000, 500,100,50,30,10,5, 1, 2, 0.5\}$, where larger values of $\beta$ correspond to increasingly deterministic behavior.
For group-specific policies, the policy class consists of $N$ separate policies, where $|\mathcal S|$ denotes the number of sensitive groups. Consequently, we independently swept all threshold and stochasticity configurations for each group. This results in $N \times (\text{number of thresholds}) \times (\text{number of }\beta\text{ values})$ policy instances (i.e., $|\mathcal S| \times 100 \times 10$ in our experiments), for group-based stochastic policies. While this configuration sweep incurs a substantial computational cost, the overall runtime can be significantly reduced through parallel processing, as individual policy evaluations are independent and can be executed concurrently. In all of our experiments, we consider $|\mathcal S|=2$ groups. Across all configurations, stakeholder utilities are computed using fixed, dataset-specific constants. Our theoretical results are not tied to these particular choices; while different constants may affect the magnitude of improvements, the key insights follow from a convexity analysis determined by the asymmetry and alignment of the utility constants.

For the computation of the hypervolume (HV) and normalized hypervolume (nHV), both the utopia and reference points are required. While the utopia point is typically defined theoretically, in practice we approximate it using the best utility values observed across the combined Pareto fronts of all evaluated policies. To ensure a fair and consistent comparison, we therefore define the nadir point using the worst utility values observed across the Pareto fronts of all policies considered jointly. Additionally, to obtain the reference point, we slightly perturb the nadir point by moving it away from the utopia point in each dimension, which helps avoid numerical issues in the HV computation.

\subsection{Synthetic Credit Dataset} \label{app:syntheticloan}

We construct a synthetic dataset modeled after a credit-lending~\cite{hofmann1994statlog} dataset. The data is generated using the causal model adapted from~\cite{karimi2020algorithmic}, consisting of seven variables that include both root causes and mediating factors. Gender ($G$) and age ($A$) serve as root nodes in the causal graph, influencing education ($E$), income ($I$), savings ($S$), loan amount ($L$) and loan duration ($D$) through complex, nonlinear, and non-additive relationships. For example, savings ($S$) depends on income ($I$), and loan amount ($L$) is further influenced by $S$, mimicking real-world financial dependencies.
The outcome variable $Y$ (repayment) is a nonlinear function of loan amount $L$, duration $D$, and the interaction between income $I$ and savings $S$. The complete data-generation mechanism (DGM) is defined as follows: 
{
\begin{equation*}
\begin{aligned}%
    &f_G : G =U_{G}, \quad   U_{G}~\sim \text{$Bern$}(\text{Bias})\\
    &f_A : A=-35+U_{A},\quad    U_{A}~\sim \operatorname{Gamma}(10,3.5)\\
    &f_E : E=-0.5+\sigma(-1+0.5 G+\sigma(0.1 A)+U_{E}),\\ & \quad \quad U_{E}\sim\mathcal{N}(0,0.25)\\
    &f_I : I=-4+0.1(A+35)+2 G+G E+U_{I}, \\ & \quad \quad
     U_{I}~\sim~\mathcal{N}(0,4)\\
    &f_S : S=-4+ 1.5 \mathds{1}\{I > 0\} I+U_{S}, 
    U_{S}~\sim~\mathcal{N}(0,5) \\
    &f_L : L= 1+0.01(A-5)(5-A)+ 2(1-G) + \beta S  +U_{L}, \\
    & \quad \quad U_{L}~\sim~\mathcal{N}(0,10), \beta \in \{0, 0.03\} \\
    &f_D : D=-1+0.1 A + 3 (1-G) +L+U_{D}, U_{D}~\sim~\mathcal{N}(0,9)\\
     &f_Y : Y=\text{$Bern$}\{\sigma(\delta (-L -D) + 0.3 (I+ S + \alpha I  S ) + U_Y)\}  \\
     \quad & \quad \quad U_{Y}~\sim~\mathcal{N}(0,2)
     \\
      & \alpha = 
\begin{cases}
1, & \text{if } \mathds{1}\{I > 0\} \land \mathds{1}\{S > 0\} = 1 \\
-1, & \text{otherwise}
\end{cases} \\
\end{aligned}
\end{equation*}
where $\mathds{1}$ denotes the indicator function.
}

Since the DGM is fully specified, we use the true probabilities to make decisions. The constants used to formulate the utility functions for both  the decision maker and the decision subject are based on the decisions ($D$) and the outcomes ($Y$):

{
\[
C_{DM} = 
\begin{array}{c|cc}
D \backslash Y & Y=0 & Y=1 \\
\hline
D=0 & 0 & 0 \\
D=1 & 0.1\cdot\mu_L = -0.4431 
    & \mu_L\cdot\mu_D\cdot r = 28.5473 \\
\end{array}
\]
}

\[
C_{DS} = 
\begin{array}{c|cc}
D \backslash Y & Y=0 & Y=1 \\
\hline
D=0 & 0 & -1 \\
D=1 & -5 & 10 \\
\end{array}
\]
where: 
\[
\mu_L := \mathrm{mean}(L), \quad 
\mu_D := \mathrm{mean}(D), \quad 
r := \mathrm{rate}.
\]
The DM and DS utility constants are inspired from~\cite{hertweck2023justice}.

\subsection{Synthetic Hiring Dataset} \label{app:synthetichiring}

This synthetic hiring dataset models candidate features and hiring outcomes using a combination of features sensitive attributes and covariates based on the framework in~\cite{binkyte2022need,binkyte_2025_16359243}. Age and gender are sensitive attributes, while education, years of experience, and previous companies serve as covariates. The interview score is considered the treatment variable influencing the hiring decision. The definitions of the variables and the corresponding data-generation mechanism are provided as follows:

\begin{align*}
   \text{Age:} \quad  U_A &\sim \text{Gamma}(10, 3.5)\ %
   \\
   \text{Gender:} \quad U_G &\sim \text{Bernoulli}(0.5) \\
   \text{Education:} \quad U_E &\sim \mathcal{N}(0, 0.25) \\
   \text{Previous Companies:} \quad U_P &\sim \text{Poisson}(2) \\
   \text{Experience Years:} \quad U_{EY} &\sim \mathcal{N}(0, 3) \\
   \text{Interview Score:} \quad U_I &\sim \mathcal{N}(0, 1)
\end{align*}

\begin{align*}
    A &:= -35 + U_A  \\
    G &:= UG \quad \text{(Gender, binary: 0 or 1)} \\
    E &:= -0.5 + \sigma(-1 + 0.5 G + \sigma(0.1  A) + U_E)  \\
    P &:= \left\lfloor 3.5 + 0.2 A - 0.25  (1 - G) + E + U_P \right\rfloor \\
   EY &:= \min\Big( \max\big(0,\ 0.7A + 15 - 5E - 2(1 - G) + U_{EY} \big),\ A + 15 \Big) \\
    I &:= 10 \cdot \sigma\left(-2.5 - 0.3  (1 - G) + 2  E + 0.2  EY + U_I \right) \quad  \\
    Y &:= \sigma\left( -4 + 0.5 I - \sigma(0.1  P) + 0.1  EY + 0.5 \cdot E \right) 
\end{align*}

Since, this is a synthetic dataset, we have access to the true outcome probabilities. Hence, we use the probabilities generated by the data-generating mechanism to make decisions. The constants associated with different stakeholders are as follows:

\[
C_{DM} = \;
\begin{array}{c|cc}
D \backslash Y & Y=0 & Y=1 \\
\hline
D=0 & 0 & 0 \\
D=1 & 2.5 & 50 \\
\end{array}\]

\[
C_{DS} = \;
\begin{array}{c|cc}
D \backslash Y & Y=0 & Y=1 \\
\hline
D=0 & 0 & 0 \\
D=1 & -4 & 8 \\
\end{array}
\]

The constants are chosen to encode three modeling assumptions. First, for the DM, the gain from a successful hire is much larger than the loss from an unsuccessful hire, reflecting the high value of a productive employee relative to the cost of a hiring mistake. Second, for DS, hiring is welfare-improving when the candidate is qualified and welfare-reducing when the candidate is not, capturing the possibility of mismatch harms. Third, the constants are scaled so that neither stakeholder’s utility trivially dominates the other, thereby producing a non-degenerate Pareto front. Thus, the values are intended to be illustrative and to induce meaningful trade-offs, rather than to represent a uniquely calibrated ground truth.

\subsection{German Credit Dataset}
In addition to synthetic datasets, we evaluate our approach on the German Credit Dataset~\cite{hofmann1994statlog}, a widely used benchmark in fairness research. The dataset consists of $1,000$ credit applications with $20$ personal and financial attributes, with a binary label indicating whether an applicant is considered creditworthy. We focus on gender as the sensitive attribute, binarized as $G=0$ (female) and $G=1$ (male). 

We obtained the dataset from the Statlog\footnote{https://archive.ics.uci.edu/dataset/144/statlog+german+credit+data} project and, following prior work~\cite{majumdar2025causal}, we restructured the personal status feature to extract binary gender categories. 
Specifically, the categories “male single” and “male divorced/separated” were grouped into one as male, while “female single” and “female divorced/separated” were grouped as female. This preprocessing enables a clearer analysis of gender-based disparities while preserving real-world decision-making structure. The dataset is particularly relevant to fairness research in financial domains, where misclassifications can lead to socially significant harms. 

For real-world datasets we decouple prediction from decision-making, as true outcome probabilities are not directly observable. Accordingly, we train a simple neural network classifier to estimate outcome probabilities, which are subsequently treated as fixed inputs to the decision-making stage. The classifier is trained using standard cross-entropy loss and outputs a scalar probability score for each instance. For the German Credit dataset, we train a lightweight single-hidden-layer neural network with $124$ hidden units to estimate outcome probabilities. Numerical features are scaled to the $(0,1)$ range to improve training stability and performance. After training for 100 epochs, the model achieved a training accuracy of 98\% and a test accuracy of 72\%. 

\[
C_{DM} = \;
\begin{array}{c|cc}
 D \backslash Y  & Y=0 & Y=1 \\
\hline
D=0 & 0 & 0 \\
D=1 & L*0.01 & L*D/12*rate \\
\end{array}
\]
\[
C_{DS} = \;
\begin{array}{c|cc}
 D \backslash Y & Y=0 & Y=1 \\
\hline
D=0 & 0 & -1 \\
D=1 & -5 & 10 \\
\end{array}
\]
The DM and DS utility constants are taken from~\cite{hertweck2023justice}. 

\subsection{Home Credit Dataset}

The Home Credit dataset was obtained from a Kaggle competition~\cite{homecredit}. Following the preprocessing pipeline described in in~\cite{majumdar2025causal}, we utilized the application training data and selected features based on the feature importance as seen in~\cite{matthyspredicting}. In particular, we adopted the features identified as most relevant by the developed random forest classifier and information entropy-based models. The selected features are :  ``EXT SOURCE 2, EXT SOURCE 3, EXT SOURCE 1, DAYS BIRTH, NAME EDUCATION TYPE, CODE GENDER, DAYS EMPLOYED, NAME INCOME TYPE, ORGANIZATION TYPE, AMT CREDIT, AMT ANNUITY, AMT GOODS PRICE, REGION POPULATION RELATIVE" and the target variable "TARGET".
After selecting this subset, all data points with missing values were removed. The variables ``DAYS BIRTH" and ``DAYS EMPLOYED", originally recorded as negative values in days, were converted to years by dividing by ${-365}$. The resulting dataset consists of $98,859$ samples with $14$ features.
For this dataset, we trained a lightweight feedforward neural network with two hidden layers of 64 and 32 units, respectively, with ReLU activation functions. The model was trained for 100 epochs, the model achieved a training accuracy of 93\% and a test accuracy of 92\%. The constants for stakeholders for formulation of the utilities are given as follows:

\[
C_{DM} = \;
\begin{array}{c|cc}
 D \backslash Y  & Y=0 & Y=1 \\
\hline
D=0 & 0 & 0 \\
D=1 & L*0.001 & A*12*15 -L  \\
\end{array}
\]
\[
C_{DS} = \;
\begin{array}{c|cc}
 D \backslash Y & Y=0 & Y=1 \\
\hline
D=0 & 0 & -10 \\
D=1 & -50 & 100 \\
\end{array}
\]

The DM and DS utility constants are scaled version of the constants ~\cite{hertweck2023justice} for German credit dataset to match the scale of the dataset. 

\subsection{MIMIC-III Sepsis Dataset}

The MIMIC-III-sepsis dataset~\cite{Hou2020} is a subset of the MIMIC-III database, containing ICU patient records from Beth Israel Deaconess Medical Center spanning 2001 to 2012. MIMIC-III-sepsis includes patients aged between 18 and 89, who spent at least 24 hours in ICU and were diagnosed with "sepsis", "severe sepsis" or "septic shock". Features with more than $20\%$ missing values were excluded, while the remaining features' missing values were imputed. A more detailed description on the data extraction process and preprocessing can be found in \cite{Hou2020}. 
 Categorical features were encoded using one-hot encoding and continuous features are rescaled to the range $(0,1)$. The final dataset consists of $4,559$ samples and $101$ features, capturing demographic information, vital signs and laboratory measurements. In this work, we focus on the treatment variable \texttt{vent} indicating whether a patient was mechanically ventilated, and the target variable \texttt{hospital\_expire\_flag} indicating in-hospital mortality.
To estimate outcome probabilities, we train a small two-hidden-layer neural network (124 and 64 units), which we regularize using dropout of 50\% probability on the output of the second hidden layer. After training for 50 epochs, the model achieved a training accuracy of 90\% and a test accuracy of 85\%. The constants for stakeholder utility formulation are defined as follows: 
\[
C_{DM} = \;
\begin{array}{c|cc}
D \backslash Y & Y=0 & Y=1 \\
\hline
D=0 & -10 & 40 \\
D=1 & -30 & 160 \\
\end{array}
\]

\[
C_{DS} = \;
\begin{array}{c|cc}
D \backslash Y & Y=0 & Y=1 \\
\hline
D=0 & -20 & 100 \\
D=1 & -40 & 100 \\
\end{array}
\]

Since utilities are not directly observed in the clinical settings, the constants are chosen to reflect plausible treatment trade-offs. For the DM, survival yields high benefit, particularly when achieved through intervention (
$C_{DM}(1,1)=160 > C_{DM}(0,1)=40)$, while mortality is costly, with higher penalty under treatment due to resource use and intervention risk ($C_{DM}(1,0) = -30 < C_{DM}(0,0) = -10$). For the patient, survival is highly valued regardless of treatment 
$C_{DS}(0,1)=C_{DS}(1,1)=100)$, while treatment imposes additional burden when unsuccessful $C_{DS}(1,0)= - 40$ vs. $C_{DS}(0,0)= - 20$. These constants encode asymmetric benefits and harms of treatment and are chosen to induce meaningful performance-fairness trade-offs.

\section{Experimental Analysis of Stochastic Policy Parameters}\label{app:stochasticity}

In this section, we further study the behavior of stochastic policies. Specifically, in  Figure~\ref{fig:app_betaheatmaps}, we plot the utility values achieved by each shared stochastic policy, color-coded by its $\beta$ value. Unlike Figure~\ref{fig:beta_heatmap_main}, however, here we include \textbf{all} shared stochastic policies evaluated in our experiments, rather than focusing only on the PFs. 

Across most datasets, and under both Egalitarian and Rawlsian utility objectives, we observe that policies naturally form clusters according to their level of stochasticity. In particular, as stochasticity increases (i.e., moving from cooler to warmer colors), these clusters become progressively smaller, indicating that the choice of threshold becomes less important. Intuitively, higher stochasticity smooths the decision boundary, reducing sensitivity to the exact threshold value. \TightenPar{1}

More importantly, under the Egalitarian utility, we observe that increasing stochasticity often leads to improved trade-offs, especially in the middle region of the PF between the two extreme (i.e., where only one utility is optimal) points. We believe that this effect arises because stochasticity helps with cases that exhibit high outcome uncertainty. 

In contrast, under the Rawlsian utility, extreme-levels of stochasticity seem to restrict the utilities, resulting in worse trade-offs. In fact, for Rawlsian objectives, the threshold parameter appears to be very important: near-deterministic policies can yield both the worst outcomes (bottom-left corner) and best (top-right corner) outcomes, depending on the chosen threshold. These behaviors are not observed for the MIMIC dataset.

Nevertheless, in all settings, the family of stochastic policies contains deterministic policies, as a special case. As a result, even in scenarios where stochasticity is detrimental, our policy sweep still recovers the optimal near-deterministic solutions.
As a side note, these observations suggest a potential opportunity for improving computational efficiency. Since thresholds matter less as policies become more stochastic, one could reduce the number of threshold values evaluated at higher stochasticity levels, while using a finer threshold grid for near-deterministic policies where threshold choice is more critical.

\begin{figure*}
    \centering
    \includegraphics[width=1\linewidth]{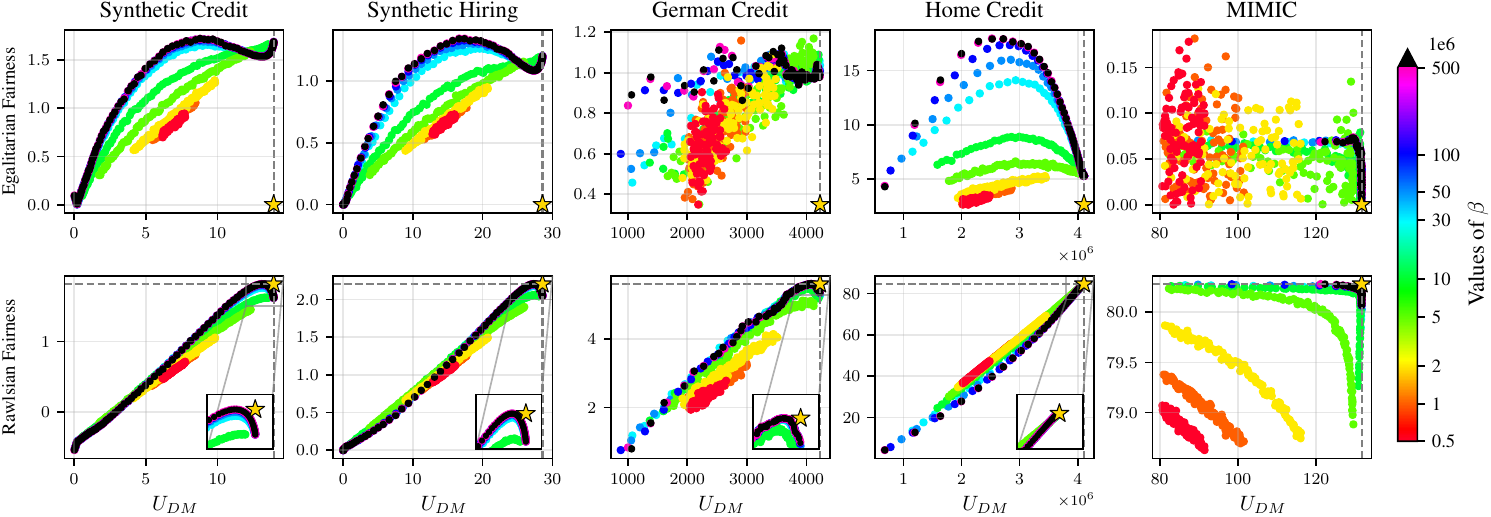}
    \caption{Shared stochastic policies (colored by $\beta$). Unlike Figure~\ref{fig:beta_heatmap_main}, all evaluated shared stochastic policies are visualized. \TightenPar{1}}
    \label{fig:app_betaheatmaps}
\end{figure*}

Finally, in Figure~\ref{fig:app-beta_decision}, we visualize the decision functions implemented by stochastic policies. These plots provide further intuition for our findings: as $\beta$ decreases and stochasticity increases, the decision functions become increasingly flat, explaining why the precise threshold plays a diminishing role in policy behavior.

\begin{figure*}
    \centering
    \includegraphics[width=0.4\linewidth]{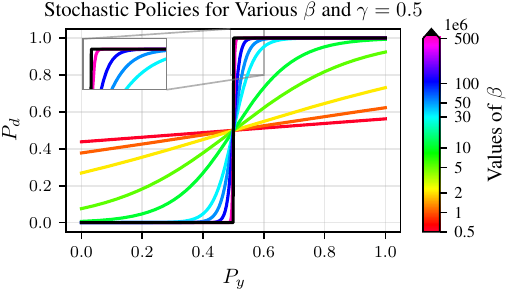}
    \caption{Decision function of stochastic policies across different $\beta$ values, for $\gamma = 0.5$.}
    \label{fig:app-beta_decision}
\end{figure*}

\section{Additional Results}
\label{app:extra_results}

In this section, we complement the results presented in Section~\ref{sec:experiments} by providing additional experimental results and visualizations that offer further insight into the behavior of the aforementioned policies.

\textbf{Hypervolume.} First, we report the (unnormalized) hypervolume (HV) values for all evaluated datasets and policy classes. These results are analogous to those shown in Table~\ref{tab:hv_auc_results}; however, instead of normalized hypervolume, we present the raw HV values. The unnormalized metric provides higher numerical resolution and reveals small differences that are sometimes obscured by normalization. In particular, in the normalized setting, deterministic and stochastic policies may occasionally appear to achieve identical HV values, whereas the unnormalized HV exposes performance gaps. The qualitative conclusions remain consistent with those of Table~\ref{tab:hv_auc_results}.

\begin{table*}[t]
\centering
\scriptsize
\begin{tabular}{ll|cc|cc|cc|cc}
\toprule

\multirow{3}{*}{\textbf{Dataset}} & \multirow{3}{*}{\textbf{Setting}} 
& \multicolumn{4}{c|}{\textbf{Egalitarian}} 
& \multicolumn{4}{c}{\textbf{Rawlsian}} \\
\cmidrule(lr){3-6}\cmidrule(lr){7-10}
& 
& \multicolumn{2}{c|}{\textbf{HV}~($\uparrow$)} 
& \multicolumn{2}{c|}{\textbf{HV$_{test}$}~($\uparrow$)} 
& \multicolumn{2}{c|}{\textbf{HV}~($\uparrow$)} 
& \multicolumn{2}{c}{\textbf{HV$_{test}$}~($\uparrow$)} \\
&
& \textbf{$PF_{det}$} & \textbf{$PF_{stoch}$} 
& \textbf{$PF_{det}$} & \textbf{$PF_{stoch}$} 
& \textbf{$PF_{det}$} & \textbf{$PF_{stoch}$} 
& \textbf{$PF_{det}$} & \textbf{$PF_{stoch}$} 
\\

\midrule
\multirow{2}{*}{Synthetic Credit} & shared & 6.14& 11.23& 6.46& 11.80 & 0.1811& 0.1814& 0.1950& 0.1954\\
     & group & 19.69& 20.90& 20.03& 21.67 & 0.1878& 0.1885& 0.2004& 0.2011\\
    \midrule
    \multirow{2}{*}{Synthetic Hiring} & shared & 8.71& 16.86& 8.44& 16.19 & 0.2264& 0.2268& 0.1999& 0.2005\\
     & group & 31.03& 32.27& 29.95& 30.97 & 0.2378& 0.2385& 0.2078& 0.2088\\
    \midrule
    \multirow{2}{*}{German} & shared & 420.97& 815.36& 285.20& 300.24 & 47.02& 47.91& 47.35& 47.67\\
     & group & 1876.82& 2094.06& 613.99& 926.37 & 57.31& 58.79& 58.25& 57.78\\
    \midrule
    \multirow{2}{*}{HomeCredit} & shared & $2.79 \times 10^{5}$& $5.74 \times 10^{6}$& $2.15 \times 10^{5}$& $5.45 \times 10^{6}$ & $1.24 \times 10^{4}$& $1.24 \times 10^{4}$& 313.33& 314.70\\
     & group & $1.79 \times 10^{7}$& $1.80 \times 10^{7}$& $1.69 \times 10^{7}$& $1.76 \times 10^{7}$ & $1.25 \times 10^{4}$& $1.25 \times 10^{4}$& 346.88& 349.00\\
    \midrule
    \multirow{2}{*}{MIMIC} & shared & 3.75& 3.84& 4.92& 4.92 & 1.91& 1.92& 13.57& 13.83\\
     & group & 3.86& 3.86& 29.16& 29.63 & 2.04& 2.05& 14.47& 14.56\\
\bottomrule
\end{tabular}
\caption{{Unnormalized HV for different policy classes in utility space.} Higher HV values indicate broader Pareto-optimal trade-off regions.}
\label{tab:app_hv_auc_results}
\end{table*}

\textbf{Test-time PFs.} Next, we visualize the Pareto fronts (PFs) at test time. To obtain these, we evaluate all policies lying on the training PF on the test set and then recompute the resulting PF based on their test-time utilities. We observe that, in some cases, the test-time PF becomes narrower compared to the training PF, reflecting distributional shift or reduced generalization capacity.

Despite this contraction, the qualitative trends observed during training largely persist at test time. In particular, stochastic policies consistently outperform, or at least match, deterministic policies, and group-based policies generally outperform shared policies. 
Additionally, for MIMIC under the Egalitarian objective, the shared deterministic Pareto front collapses to a single point, further highlighting the limited expressivity of shared deterministic policies.

\begin{figure*}
    \centering
    \includegraphics[width=0.9\linewidth]{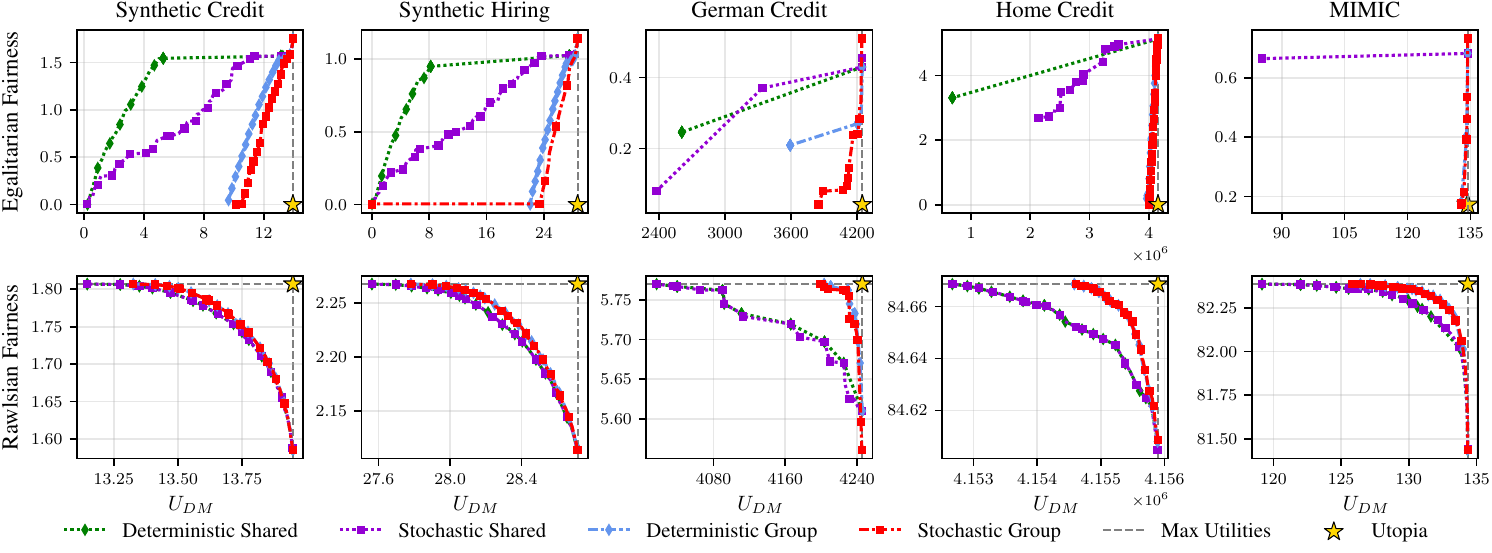}
    \caption{Test-time Pareto fronts, obtained by evaluating training Pareto-optimal policies on the test set and comparing deterministic and stochastic policies across performance-fairness trade-offs.}
    \label{fig:app_test_pareto}
\end{figure*}

\textbf{Fairness Gain.} Finally, in Figure~\ref{fig:app_aucgains} we visualize the point-wise fairness gain achieved by stochastic policies. This analysis requires interpolating between Pareto-optimal points, as deterministic policies often produce narrower or discontinuous PFs. As a result, \textbf{interpolation can occasionally yield negative fairness gains, which should be interpreted as artifacts of this approximation rather than true degradations in fairness}.

Overall, the trends mirror those observed in the main experiments. In the shared-policy setting, stochastic policies lead to improved fairness, particularly in regions where fairness is prioritized over the decision maker’s utility. In contrast, in the group-policy setting, stochasticity primarily improves fairness near the decision maker’s optimal utility. These results further reinforce the advantages of stochastic policies across different performance-fairness trade-offs.

\begin{figure*}
    \centering
\includegraphics[width=0.9\linewidth]{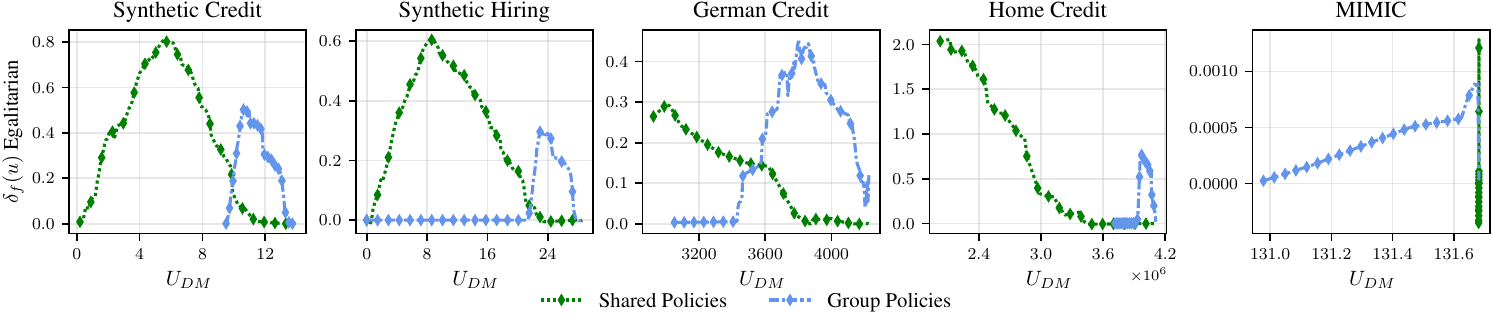}
    \caption{Fairness gain under Egalitarian justice achieved by switching from deterministic to stochastic policies, for both shared and group-specific policies at training time.}
    \label{fig:app_aucgains}
\end{figure*}

\textbf{Ablation on symmetry and alignment of utilities.}
Figure~\ref{fig:app_HVgainvsRatio} illustrates how hypervolume gains from stochastic policies vary with utility asymmetry and alignment. We compare \emph{aligned} and \emph{misaligned} settings that differ only in the DS utility constants, while sweeping the ratio $|C_{11}^{DM}| / |C_{10}^{DM}|$.
Across all panels, gains are negligible when $|C_{11}^{DM}| \leq |C_{10}^{DM}|$, indicating that stochasticity provides no benefit in low-asymmetry regimes. As the ratio increases beyond 1, clear improvements emerge.
Under the \emph{Egalitarian} objective, hypervolume gains grow steadily with asymmetry in both aligned and misaligned settings, and are nearly identical across them. This confirms that Egalitarian improvements depend primarily on asymmetry. Moreover, shared policies consistently achieve larger gains than group policies.
In contrast, under the \emph{Rawlsian} objective, gains are strongly conditioned on alignment. In the aligned case, improvements remain negligible across most ratios, with only minor increases at extreme asymmetry. However, in the misaligned setting, gains increase sharply with asymmetry, particularly for large ratios, and again are more pronounced for shared policies.
Overall, the figure highlights a clear separation: asymmetry alone is sufficient to induce gains under Egalitarianism, while Rawlsian improvements require both asymmetry and misalignment. These trends visually corroborate Corollaries~\ref{cor:egal} and~\ref{cor:rawls}.
\begin{figure}
    \centering
    \includegraphics[width=0.65\linewidth]{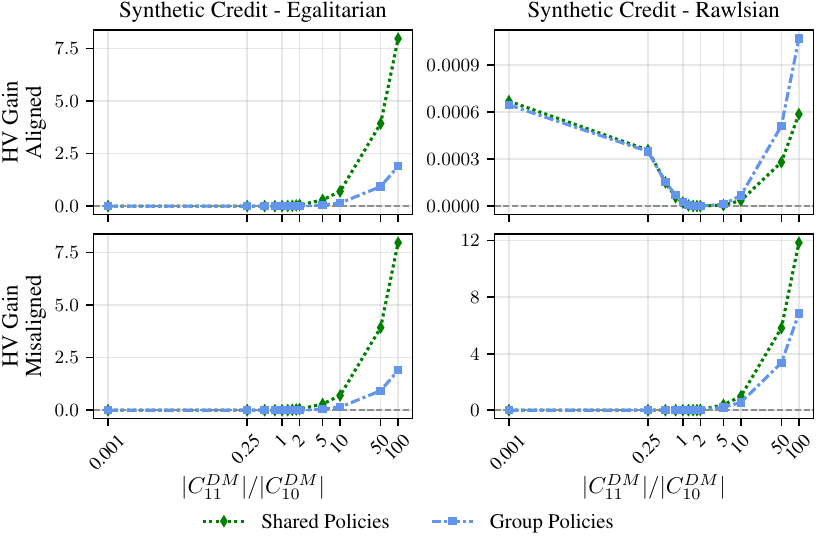}
    \caption{Hypervolume gain of stochastic policy over deterministic policies as a function of the utility symmetry ratio $|C_{11}^{DM}| / |C_{10}^{DM}|$ for both aligned and misaligned cases on the synthetic credit dataset. This figure complements Fig.~\ref{fig:HVgainVSratio}.}
    \label{fig:app_HVgainvsRatio}
\end{figure}

\end{document}